\documentclass{article}

    \usepackage[final,nonatbib]{neurips_2020}

\usepackage[utf8]{inputenc} 
\usepackage[T1]{fontenc}    

\usepackage[colorlinks=true,allcolors=blue]{hyperref}%

\usepackage{url}            
\usepackage{booktabs}       
\usepackage{amsfonts}       
\usepackage{nicefrac}       
\usepackage{microtype}      

\usepackage{amsmath,mathtools}

\usepackage{amsthm}
\usepackage{xcolor}

\newcommand{\deletethis}[1]{ { \color{red} delete :{#1}}  }

\newtheorem{lemma}{Lemma}
\newtheorem{theorem}{Theorem}

\newtheorem{assumption}{AS.}

\newcommand{\va}{{\mathbf {a}}}
\newcommand{\vg}{{\mathbf {g}}}
\newcommand{\vh}{{\mathbf {h}}}

\newcommand{\vs}{{\mathbf {s}}}
\newcommand{\vv}{{\mathbf {v}}}
\newcommand{\vx}{{\mathbf {x}}}
\newcommand{\vy}{{\mathbf {y}}}

\newcommand{\homovec}[1]{{(#1, \hspace{2pt} 1)}}

\usepackage[square,numbers]{natbib}

\usepackage{algorithm}
\usepackage[noend]{algorithmic}

\usepackage{enumitem} 
\usepackage{graphicx}

\usepackage{caption}
\usepackage{subcaption}

\title{Practical {Quasi-Newton} Methods for Training Deep Neural Networks
}

\author{%
  Donald Goldfarb, Yi Ren, Achraf Bahamou 
    \\
  Department of Industrial Engineering and Operations Research \\
  Columbia University\\
  New York, NY 10027 \\
  \texttt{ \{goldfarb, yr2322, ab4689\}@columbia.edu} \\
}

\begin{document}

\maketitle


\begin{abstract}

We consider the development of practical stochastic quasi-Newton, and in particular  Kronecker-factored block-diagonal BFGS and L-BFGS methods, for training deep neural networks (DNNs).
In DNN training, the number of variables and components of the gradient $n$ is often of the order of tens of millions and the Hessian has $n^2$ elements. Consequently, computing and storing a full $n \times n$ BFGS approximation or storing a modest number of (step, change in gradient) vector pairs for use in an L-BFGS implementation is out of the question. In our proposed methods, we approximate the Hessian by a block-diagonal matrix and use the structure of the gradient and Hessian to further approximate these blocks, each of which corresponds to a layer, as the Kronecker product of two much smaller matrices. This is analogous to the approach in KFAC \cite{martens2015optimizing}, which computes a Kronecker-factored block-diagonal approximation to the Fisher matrix in a stochastic natural gradient method. Because of the indefinite and highly variable nature of the Hessian in a DNN, we also propose a new damping approach to keep the upper as well as the lower bounds of the BFGS and L-BFGS approximations bounded. In tests on autoencoder feed-forward neural network models with either nine or thirteen layers applied to three datasets, our methods outperformed or performed comparably to KFAC and state-of-the-art first-order stochastic methods.
\end{abstract}

\section{Introduction}
We consider in this paper the development of practical stochastic quasi-Newton (QN), and in particular Kronecker-factored block-diagonal BFGS \cite{broyden1970convergence,fletcher1970new,goldfarb1970family,shanno1970conditioning} and L-BFGS \cite{liu1989limited}, methods for training  deep neural networks (DNNs). 
 Recall that the BFGS method starts each iteration with a symmetric positive definite matrix $B$ (or $H= B^{-1}$) that approximates the current Hessian matrix (or its inverse), computes the gradient
$\mathbf{\nabla f}$ of $f$ at the current iterate $\mathbf{x}$ and then takes a step
$\mathbf{s} = - \alpha H \mathbf{\nabla f}$, where $\alpha$ is a step length (usually) determined by some inexact line-search procedure, such that
$\mathbf{y}^\top \mathbf{s} > 0$, where
$\mathbf{y} = \mathbf{\nabla f}^+ - \mathbf{\nabla f}$ and $\mathbf{\nabla f}^+$ is the gradient of $f$ at the new point $\mathbf{x}^+ = \mathbf{x} + \mathbf{s}$.
The method then computes 
an updated approximation $B^+$ to $B$ (or $H^+$ to $H$) that remains symmetric and positive-definite and satisfies the so-called {\it quasi-Newton} (QN) condition $B^+ \mathbf{s} =\mathbf{y}$ (or equivalently, $H^+ \mathbf{y} =\mathbf{s}$).
A consequence of this is that the matrix $B^+$ operates on the vector $\mathbf{s}$ in exactly the same way as the average of the Hessian matrix along the line segment between $\mathbf{x}$ and $\mathbf{x}^+$ operates on $\mathbf{s}$.

In DNN training, the number of variables and components of the gradient $n$ is often of the order of tens of millions and the Hessian has $n^2$ elements. Hence, computing and storing a full $n \times n$ BFGS approximation or storing $p$ $(\mathbf{s}, \mathbf{y} )$ pairs, where $p$ is approximately $10$ or larger for use in an L-BFGS implementation, is out of the question.
Consequently, in our methods, we approximate the Hessian by a block-diagonal matrix,
where each diagonal block corresponds to a layer,
further approximating them
as the Kronecker product of two much smaller matrices, as
in
\cite{martens2015optimizing,botev2017practical,gupta2018shampoo,dangel2019modular}.

\textbf{Literature Review on Using Second-order Information for DNN Training. }
For solving the stochastic optimization problems with high-dimensional data that arise in machine learning (ML), stochastic gradient descent (SGD) \cite{robbins1951stochastic} and its variants are the methods that are most often used, especially for training DNNs.
These variants include such methods as AdaGrad \cite{duchi2011adaptive}, RMSprop \cite{hinton2012neural}, and Adam \cite{kingma2014adam}, all of which scale the stochastic gradient by a diagonal matrix based on estimates of the first and second moments of the individual gradient components. Nonetheless, there has been a lot of effort to find ways to take advantage of second-order information in solving ML optimization problems. Approaches have run the gamut from the use of a diagonal re-scaling of the stochastic gradient, based on the secant condition associated with quasi-Newton (QN) methods \cite{bordes2009sgd}, to sub-sampled Newton methods (e.g. see \cite{xu2019newton}, and references therein), including those that solve the Newton system using the linear conjugate gradient method (see \cite{byrd2011use}).

In between these two extremes are stochastic methods that are based either on QN methods or generalized Gauss-Newton (GGN) and natural gradient \cite{amari2000adaptive}
methods. For example, a stochastic
L-BFGS method for solving strongly convex problems was proposed in
\cite{byrd2016stochastic} that uses sampled Hessian-vector products rather than gradient differences, which was proved in
\cite{moritz2016linearly} to be linearly convergent by incorporating the variance reduction technique (SVRG \cite{johnson2013accelerating}) to alleviate the effect of noisy gradients. A closely related variance reduced block L-BFGS method was proposed in \cite{gower2016stochastic}.
A regularized stochastic BFGS method was proposed in \cite{mokhtari2014res}, and an online L-BFGS method {was proposed} in
\cite{mokhtari2015global} for strongly convex problems and extended in \cite{lucchi2015variance} to incorporate SVRG variance reduction. Stochastic BFGS and L-BFGS methods were also developed for online convex optimization in \cite{schraudolph2007stochastic}. For nonconvex problems, a damped L-BFGS method which incorporated SVRG variance reduction was developed and its convergence properties was studied in \cite{wang2017stochastic}.

GGN methods that approximate the Hessian have been proposed, including the Hessian-free method \cite{martens2010deep} and the Krylov subspace method \cite{vinyals2012krylov}. Variants of the closely related natural gradient method
that use block-diagonal approximations to the Fisher information matrix, where blocks correspond to
layers, have been proposed in e.g. 
\cite{heskes2000,desjardins2015natural,martens2015optimizing,fujimoto2018neural}. Using further approximation of each of these (empirical) Fisher matrix and GGN blocks by the Kronecker product of two much smaller matrices, the efficient KFAC \cite{martens2015optimizing}, KFRA \cite{botev2017practical}, EKFAC \cite{george2018fast}, and Shampoo \cite{gupta2018shampoo}  methods were developed.
See also
\cite{Ba2017DistributedSO}
and \cite{dangel2019modular}, 
\cite{le2010fast}, which combine both Hessian and covariance (Fisher-like) matrix information in  stochastic Newton type methods,
Also, methods are given in \cite{li2018natural, wang2020information}  that replace the Kullback-Leibler divergence by the Wasserstein distance to define the natural gradient, but with a greater computational cost.



\textbf{Our Contributions. }
The main contributions of this paper can be summarized as follows:
\begin{enumerate}[topsep=0pt,itemsep=-1ex,partopsep=1ex,parsep=1ex,wide]
     \item 
     New BFGS and limited-memory variants (i.e. L-BFGS) that take advantage of the structure of feed-forward DNN training problems;
    \item Efficient non-diagonal second-order algorithms for deep learning that require a comparable amount of memory and computational cost per iteration as first-order methods;
    \item A new damping scheme for BFGS and L-BFGS updating of an inverse Hessian approximation, that not only preserves its positive definiteness, but also limits the decrease (and increase) in its smallest (and largest) eigenvalues for non-convex problems;
    \item A novel application of Hessian-action BFGS;
    \item The first proof of convergence (to the best of our knowledge) of a stochastic Kronecker-factored
    {quasi-Newton}
    method.
\end{enumerate}


\section{Kronecker-factored Quasi-Newton Method for DNN}


After reviewing the computations used in
DNN training, 
we describe the Kronecker structures of the gradient and Hessian for a single data point, followed by their extension to approximate expectations of these quantities for multiple data-points and give a generic algorithm that employs
BFGS (or L-BFGS) approximations for the Hessians.

\textbf{Deep Neural Networks.}
We consider a feed-forward DNN with $L$ layers, defined by weight matrices $W_l$ (whose last columns are bias vectors $b_l$), activation functions $\phi_l$ for $l \in \{ 1\dots L\}$ and loss function $\mathcal{L}$.
For a data-point $(x, y)$, the loss $\mathcal{L} \left(a_{L}, y\right)$ between the output $a_{L}$ of the DNN and $y$ is a non-convex function of $\theta= \left[\operatorname{vec}\left(W_{1}\right)^{\top},  \ldots, \operatorname{vec}\left(W_{L}\right)^{\top}\right]^{\top}$. 
The network's forward and backward pass for a single input data point $(x, y)$ is described in Algorithm \ref{feedforward}.

\begin{algorithm}[ht]
    \caption{Forward and backward pass of DNN {for a single data-point}}
    \label{feedforward}
    \begin{algorithmic}[1]
    
    \STATE Given input $(x,y)$, 
    weights (and biases) $W_l$, and activations $\phi_l$ for $l \in [1, L]$
    
    
    \STATE
    $\va_0 = x$;
    {\bf for}
    {$l = 1, .., L$}
    {\bf do}
    $\bar{\va}_{l-1} = \homovec{\va_{l-1}}$; 
    $\vh_{l} = {W}_l \bar{\va}_{l-1} $;
    $\va_{l} = \phi_l(\vh_{l})$

     \STATE $\left.\mathcal{D} \va_{L} \leftarrow \frac{\partial \mathcal{L}(z, y)}{\partial z}\right|_{z=\va_{L}}$

    \STATE
    {
    {\bf for}
    {$l = L, .., 1$}
    {\bf do}
    $\vg_{l} = \mathcal{D} \va_{l} \odot \phi_{l}^{\prime}(\vh_{l})$;
    $\mathcal{D} {W}_l = \vg_{i} \bar{\va}_{i-1}^{\top}$;
    $\mathcal{D} \va_{l-1} = W_{l}^{\top} \vg_{l}$
    }
    
    \end{algorithmic}
\end{algorithm}

For a training dataset that contains multiple data-points indexed by
{$i = 1, ..., I$}, let
{$f(i; \theta)$} denote the loss for the
{$i$th} data-point. Then, viewing the dataset as an empirical distribution, the total loss function $f(\theta)$ that we wish to minimize is
$$
    f(\theta) := \mathbb{E}_i [f(i; \theta)] := \frac{1}{I} \sum_{i=1}^I f(i; \theta).
$$


\textbf{Single Data-point: Layer-wise Structure of the Gradient and Hessian. }
Let ${\mathbf{\nabla f}}_l$ and $\nabla^2 f_l$ denote, respectively, the restriction of $\mathbf{\nabla f}$ and $\nabla^2 f$ to the weights $W_l$ in layer $l = 1, \ldots,L$. 
For a single data-point 
${\mathbf{\nabla f}}_l$ and $\nabla^2 f_l$ 
have a tensor (Kronecker) structure,
as shown in \cite{martens2015optimizing} and \cite{botev2017practical}. 
%
Specifically, 
%
%
\begin{align}
& {\mathbf{\nabla f}}_l(i)
= \vg_l(i) (\va_{l-1}(i))^\top,
\quad
{\rm equivalently,}
\quad
\text{vec}({\mathbf{\nabla f}}_l(i)) =  \va_{l-1}(i) \otimes \vg_l(i),
\label{eq_7}
\\
& 
    \nabla^2 f_l(i) = (\va_{l-1}(i) (\va_{l-1}(i))^\top) \otimes G_l(i),
    \label{eq_8}
\end{align}
where the {pre-activation} gradient $\vg_l(i) = \frac{\partial f(i)}{\partial \vh_l(i)}$,
and the {pre-activation} Hessian
$G_l(i) = \frac{\partial^2 f(i)}{\partial \vh_l(i)^2}$.
Our algorithm uses an approximation to   
{$(G_l(i))^{-1}$},
which is updated via the BFGS  updating formulas based upon a secant condition that relates
the change in $\vg_l(i)$ with the change in $\vh_l(i)$.


Although we focus on fully-connected layers in this paper, the idea of Kronecker-factored approximations to the diagonal blocks $\nabla^2 f_l$, $l = 1,\ldots,L$ of the Hessian can be extended to other layers used in deep learning, such as convolutional and recurrent layers.


\textbf{Multiple Data-points: Kronecker-factored QN Approach. }
Now consider the case where we have a dataset of
{$I$} data-points indexed by
{$i = 1, \ldots, I$}. By (\ref{eq_8}), we have
\begin{align}
    \mathbb{E}_i [\nabla^2 f_l (i)]
   \approx \mathbb{E}_i \left[ \va_{l-1}(i) (\va_{l-1}(i))^\top \right] \otimes \mathbb{E}_i \left[ G_l(i) \right]
    := A_l \otimes G_l
    \label{eq_12}
\end{align}
Note that the approximation in (\ref{eq_12}) that the expectation of the Kronecker product of two matrices equals the Kronecker product of their expectations is the same as the one used by KFAC \cite{martens2015optimizing}. Now, based on this structural approximation, we use
$H^l = H^l_a \otimes H^l_g$
as our QN approximation to
{$\left( \mathbb{E}_i [\nabla^2 f_l (i)] \right)^{-1}$}, where $H^l_a $ and $H^l_g$ are positive definite approximations to  $A_l^{-1}$ and $G_l^{-1}$, respectively.
Hence, using our layer-wise block-diagonal approximation to the Hessian, a step in our algorithm for each layer $l$ is computed as
\begin{align}
   \text{vec}(W_l^+) -    \text{vec}(W_l)
    = - \alpha H^l \text{vec} \left( \widehat{{\mathbf{\nabla f}}_l} \right)
    = - \alpha (H^l_a \otimes H^l_g) \text{vec} \left( \widehat{{\mathbf{\nabla f}}_l} \right)
    = - \alpha \text{vec} \left( H^l_g \widehat{{\mathbf{\nabla f}}_l} H^l_a \right),
    \label{eq_1}
\end{align}
where $\widehat{\mathbf{\nabla f}_l}$ denotes the estimate to
{$\mathbb{E}_i [\mathbf{\nabla f}_l(i)]$} and $\alpha$ is the learning rate. After computing $W_l^+$ and 
performing another forward/backward pass, our method
computes or updates $H_a^l$ and $H_g^l$ as follows:

\begin{enumerate}[topsep=0pt,itemsep=-1ex,partopsep=1ex,parsep=1ex,wide]
    \item 
    For $H_g^l$, we use a damped version of BFGS (or L-BFGS) (See Section \ref{sec_1}) based on the 
    $(\vs, \vy)$ pairs corresponding to the average change in
    {$\vh_l(i)$}  
    and in the gradient with respect to
    {$\vh_l(i)$}; i.e.,
    \begin{align}
        \vs_g^l = \mathbb{E}_i [\vh_l^+(i)] - \mathbb{E}_i [\vh_l(i)],
        \qquad
        \vy_g^l = \mathbb{E}_i [\vg_l^+(i)] - \mathbb{E}_i [\vg_l(i)].
        \label{eq_11}
    \end{align}

    \item For $H_a^l$ 
    we use the "Hessian-action" BFGS method described in Section \ref{sec_4}. The issue of possible singularity of the positive semi-definite matrix $A_l$ approximated by $(H_a^l)^{-1}$
    is also addressed there by incorporating a {\bf Levenberg-Marquardt (LM)} damping term.
\end{enumerate}



 Algorithm \ref{klbfgs} gives a high-level summary of 
our {\bf K-BFGS} or {\bf K-BFGS(L)} algorithms (which use BFGS or L-BFGS to update $H_g^l$, respectively).
See Algoirthm \ref{algo_3} in the appendix for
a detailed pseudocode.
The use of mini-batches is described in Section \ref{sec_17}.
Note that, an additional forward-backward pass is used in Algorithm \ref{klbfgs} because the quantities in (\ref{eq_11}) need to be estimated using the same mini-batch.

\begin{algorithm}
    \caption{
    {High-level summary of}
    K-BFGS / K-BFGS(L)
    }
    \label{klbfgs}
    \begin{algorithmic}[1]

    \REQUIRE
    Given {initial weights} $\theta$, batch size $m$, learning rate
    {$\alpha$}

    \FOR {$k=1,2,\ldots$}
        \STATE Sample mini-batch of size {$m$}: $M_k = \{\xi_{k,i}, i =1, \dots, m\}$
        
        \STATE Perform a forward-backward pass over the current mini-batch $M_k$ (see Algorithm \ref{feedforward})
        
            
        
        {
        \STATE
        \label{line_6}
        {\bf for}
        $l=1, \ldots, L$
        {\bf do}
             $p_l = H_g^l \widehat{\mathbf{\nabla f}}_l H_a^l$;
            $W_l = W_l - \alpha \cdot p_{l}$
        }
        
        \STATE Perform another forward-backward pass over $M_k$ to get $(\vs_g^l, \vy_g^l)$
        \label{line_1}
        
        \STATE Use {damped} BFGS or L-BFGS to update $H_g^l$ ($l = 1, ..., L$) (see Section \ref{sec_1}, {in particular Algorithm \ref{algo_2}})
        
        \STATE Use Hessian-action BFGS
        to update $H_a^l$ ($l = 1, ..., L$) {(see Section \ref{sec_4}
        )}

        \ENDFOR  
    
    \end{algorithmic}
\end{algorithm}

\section{{BFGS and L-BFGS for \texorpdfstring{$G_l$}{TEXT}}}
\label{sec_1}

\textbf{Damped BFGS Updating.}
It is well-known that training a DNN is a 
non-convex optimization problem. As (\ref{eq_8}) and (\ref{eq_12}) show, this non-convexity  manifests in the fact that $G_l \succ 0$ often does not hold. 
Thus, for the BFGS update of $H_g^l$, the approximation to $G_l^{-1}$, to remain positive definite, we have to
ensure that $(\vs_g^l)^\top \vy_g^l > 0$. 
Due to the stochastic setting, ensuring this condition
by line-search, as is done in deterministic settings, is impractical.
In addition, due to the large changes in curvature in DNN models that occur as the parameters are varied, we also need to suppress large changes to  $H_g^l$ as it is updated. 
To deal with both of these issues,
we propose a \textbf{double damping (DD)} procedure (Algorithm \ref{algo_2}), which is based upon {\bf Powell's damped-BFGS} approach \cite{powell1978algorithms}, for modifying the $(\vs_g^l, \vy_g^l)$ pair.
To motivate Algorithm \ref{algo_2}, consider the formulas used for BFGS updating of $B$ and $H$:
\begin{align}
    B^+ = B - \frac{B \vs \vs^\top B}{\vs^\top B \vs} + \rho \vy \vy^\top,
    \quad
    H^+ = (I - \rho \vs \vy^\top) H (I - \rho \vy \vs^\top) + \rho \vs \vs^\top,
    \label{eq_24}
\end{align}
where $\rho = \frac{1}{\vs^\top \vy} > 0$.   
If we can ensure that
{$0 < \frac{\vy^\top H \vy}{\vs^\top \vy} \leq \frac{1}{\mu_1}$}
and
{$0 < \frac{\vs^\top\vs}{\vs^\top \vy} \leq \frac{1}{\mu_2}$}
, then we can obtain the following bounds:
{
\begin{align}
    \| B^+\| & \leq \|B - \frac{B \vs \vs^\top B}{\vs^\top B \vs}\| + \|\rho \vy \vy^\top \| 
     \leq \|B\| + \| \frac{ B^{1/2}H^{1/2} \vy \vy^\top H^{1/2}B^{1/2}}{\vs^\top \vy} \|
    \\
    & \leq \|B\| + \|B\| \frac{ \|H^{1/2} \vy \|^2}{\vs^\top \vy} 
    \leq \|B\| \left( 1 + \frac{\vy^\top H \vy} {\vs^\top \vy} \right)
    \leq \|B\| \left( 1 + \frac{1}{\mu_1}\right) 
    \label{eq_21}
\end{align}
}
and 
{
\begin{align}
    \|H^+ \| 
    & \leq \| H^{1/2} - 
\frac{ \vs \vy^\top H^{1/2} }{ \vs^\top\vy}\|^2 + \|\frac{\vs \vs^\top} 
    {\vs^\top \vy} \|
    \leq \left( \| H^{1/2}\| + \frac{ \|\vs\| \|H^{1/2}\vy\|}{ \vs^\top\vy} \right)^2 + \frac{\|\vs\|^2} {\vs^\top \vy}
    \\
    &\leq \left( \| H^{1/2}\| + (\frac{ \vs^\top\vs} {\vs^\top \vy} )^{1/2} 
             ( \frac{\vy^\top H \vy}{ \vs^\top\vy} )^{1/2} \right)^2 + \frac{\vs^\top\vs} {\vs^\top \vy}
             \leq \left(\| H^{1/2}\| + \frac{1}{\sqrt{\mu_1 \mu_2}}\right)^2 +  \frac{1}{\mu_2}.
    \label{eq_22}
\end{align}
}

Thus, the change in $B$ (and $H$) is controlled if
{$\frac{\vy^\top H \vy}{\vs^\top \vy} \leq \frac{1}{\mu_1}$}
and
{$\frac{\vs^\top\vs}{\vs^\top \vy} \leq \frac{1}{\mu_2}$}.
Our DD approach is a two-step procedure, where the first step (i.e. Powell's damping of $H$) guarantees that
{$\frac{\vy^\top H \vy}{\vs^\top \vy} \leq \frac{1}{\mu_1}$} 
and the second step (i.e., {Powell's damping with $B=I$}) guarantees that
{$\frac{\vs^\top\vs}{\vs^\top \vy} \leq \frac{1}{\mu_2}$}.
Note that there is no guarantee of
{$\frac{\vy^\top H \vy}{\vs^\top \vy} \leq \frac{1}{\mu_1}$} after the second step. However, we can skip updating $H$
in this case so that the bounds on these matrices hold.
In our implementation, we always do the update, since in empirical testing, we observed that at least 90\% of the pairs
satisfy $\frac{\vy^\top H \vy}{\vs^\top \vy} \leq \frac{2}{\mu_1}$.
See Section \ref{sec_5} in the appendix for more details on damping.

\begin{algorithm}
    \caption{Double Damping (DD)}
    \label{algo_2}
    \begin{algorithmic}[1]
    
    \STATE
    {\bf Input:}
    $\vs$, $\vy$;
    {\bf Output:}
    $\tilde{\vs}$, $\tilde{\vy}$;
    {\bf Given:}
    $H, \mu_1, \mu_2$

    \STATE
    \label{line_7}
    {\bf if}
    {$\vs^\top \vy < \mu_1 \vy^\top H \vy$}
    {\bf then}
    {$\theta_1 = \frac{(1-\mu_1) \vy^\top H \vy}{\vy^\top H \vy - \vs^\top \vy}$}
    {\bf else}
    $\theta_1 = 1$

    \STATE 
    \label{line_8}
    $\tilde{\vs} = \theta_1 \vs + (1-\theta_1) H \vy$
    \COMMENT{Powell's damping on $H$}

    \STATE
    \label{line_9}
    {\bf if}
    {$\tilde{\vs}^\top \vy < \mu_2 \tilde{\vs}^\top \tilde{\vs}$}
    {\bf then}
    $\theta_2 = \frac{(1-\mu_2) \tilde{\vs}^\top \tilde{\vs}}{\tilde{\vs}^\top \tilde{\vs} - \tilde{\vs}^\top \vy}$
    {\bf else}
    $\theta_2 = 1$
    
    \STATE 
    \label{line_10}
    $\tilde{\vy} = \theta_2 \vy + (1-\theta_2) \tilde{\vs}$
    \COMMENT{Powell's damping with $B=I$}

    \RETURN $\tilde{s}$, $\tilde{y}$
    
    \end{algorithmic}
\end{algorithm}

\textbf{L-BFGS Implementation. }
L-BFGS can also be used to update $H_g^l$. 
However, implementing L-BFGS using the standard "two-loop recursion" (see Algorithm 7.4 in \cite{nocedal2006numerical}) is not efficient. This is because the main work in computing 
$H^l_g \widehat{{\mathbf{\nabla f}}_l} H^l_a$ in 
line \ref{line_6} of Algorithm \ref{klbfgs} would require 
$4p$ 
matrix-vector
multiplications, each requiring $O(d_i d_o)$ operations, where $p$ denotes the number of $(\vs,\vy)$ pairs stored by L-BFGS. (Recall that $\widehat{{\mathbf{\nabla f}}_l} \in R^{d_o \times d_i}$.) 
Instead, we use a "non-loop" implementation \cite{byrd1994representations} of L-BFGS, whose main work involves
2 matrix-matrix multiplications, each requiring $O(p d_i d_o)$ operations.
When $p$ is not small (we used $p=100$ in our tests), and $d_i$ and $d_o$ are large, this is much more efficient, especially on GPUs.


\section{"Hessian action" BFGS for \texorpdfstring{$A_l$}{TEXT}}
\label{sec_4}

In addition to approximating $G_l^{-1}$ by $H_g^l$ using BFGS, we also propose approximating $A_l^{-1}$ by $H_a^l$ using BFGS. Note that $A_l$ does not {correspond to some} Hessian {of the objective function}. However, we can generate $(\vs, \vy)$ pairs {for it} by "Hessian action" (see e.g. \cite{byrd2016stochastic,gower2016stochastic,gower2017randomized}). 

\textbf{Connection between Hessian-action BFGS and Matrix Inversion.}
In our methods, we choose
{$\vs = H^l_a \cdot \mathbb{E}_i [\va_{l-1}(i)]$} and $\vy = A_l \vs$, which as we now show, is closely connected to using the Sherman-Morrison modification formula to invert $A_l$.
%
%
%
%
%
In particular, suppose that $A^+ = A + {c \cdot} \va \va^\top$; i.e., only a rank-one update is made to $A$. 
This corresponds to the case where {the information of $A$ is accumulated from iteration to iteration, and} the size of the mini-batch is 1 or $\va$ represents the average of the vectors
{$\va(i)$}
{from multiple data-points}. 

%

\begin{theorem}
\label{thm_1}
Suppose that $A$ and $H$ are symmetric and positive definite, and that $H = A^{-1}$. If we choose $\vs = H \va$ and $\vy = A^+ \vs$, where
$A^+ = A + c \cdot \va \va^\top$($c > 0$).
Then, the $H^+$ generated by any QN update in the \textbf{Broyden family}
\begin{align}
    H^+
    = H - \sigma H \vy \vy^\top {H} + \rho \vs \vs^\top + \phi (\vy^\top H \vy) \vh \vh^\top, \label{eq_{4.1}}
\end{align}
where $\rho = 1/\vs^\top \vy$, $\sigma = 1/\vy^\top H \vy$, $\vh = \rho \vs - \sigma H \vy$ and $\phi$ is a scalar parameter in $[0, 1]$,
equals $(A^+)^{-1}$. Note that $\phi = 1$ yields the BFGS update {(\ref{eq_24})} and $\phi = 0$ yields the DFP update.
\end{theorem}

\begin{proof}
If $\vs = H \va$ and $\vy = A^+ \vs$, then $\vh =0$ 
, so all choices of $\phi$ yield the same matrix $H^+$. 
Since $H^+ A^+ \vs = H^+ \vy = \vs$ and for any vector $\vv $ that is orthogonal to $\va$, $H^+ A^+ \vv = H^+ A \vv = \vv$, since $\vs^\top A \vv = 0$ and $\vy^\top H A \vv = 0$, 
it follows that $H^+ A^+ = I$, using the fact that $\vs$ together with any linearly independent set of $n-1$ vectors orthogonal to $\va$ spans $R^n$. (Note that $\vs^\top \va = \va^\top H \va > 0$, since
{$H \succ 0 \Rightarrow $ that $\vs$ is not orthogonal to $\va$}.)
\end{proof}

In fact, all updates in the Broyden family are equivalent to applying the Sherman-Morrison modification formula to $A^+ = A + {c \cdot} \va \va^\top$,
given $H = A^{-1}$, since after substituting for $\vs$ and $\vy$ in (\ref{eq_{4.1}}) and simplifying, one obtains 
$$ H^+ = H - H \va (c^{-1} + \va^\top H \va)^{-1} \va^\top H.$$

When using momentum, $ A^+ = \beta A + (1-\beta) \va \va^\top$ ($0 < \beta < 1$). Hence, if we still want Theorem \ref{thm_1} to hold, we have to scale $H$ by $1/\beta$ before updating it. This, however, turns out to be unstable. Hence, in practice, we use the non-scaled version of "Hessian action" BFGS.




\textbf{Levenberg-Marquardt Damping for $A_l$.
}
Since {$A_l = \mathbb{E}_i \left[ (\va_{l-1}(i) (\va_{l-1}(i))^\top) \right] \succeq 0$}
may not be positive definite, or may  have very small positive eigenvalues, we add an {\bf Levenberg-Marquardt (LM) damping} term to make our "Hessian-action" BFGS stable; i.e., we use ${A_l} + \lambda_A I_A$ instead of ${A_l}$, when we {update} $H_a^l$.
{
Specifically, "Hessian action" BFGS for $A_l$ is performed as
\begin{enumerate}[topsep=0pt,itemsep=-1ex,partopsep=1ex,parsep=1ex]
    \item
    $A_l = \beta \cdot A_l + (1-\beta) \cdot \mathbb{E}_i \left[ \va_{l-1}(i) \va_{l-1}(i)^\top \right]$; $A_l^{\text{LM}} = A_l + \lambda_A I_A$.
    
    \item
    {$\vs_a^l = H_a^l \cdot \mathbb{E}_i [\va_{l-1}(i)]$},
    $\vy_a^l = A_l^{\text{LM}} \vs_a^l$; use BFGS with $(\vs_a^l, \vy_a^l)$ to update $H_a^l$.
\end{enumerate}
}




\section{Convergence Analysis
}

\newcommand{\R}{\mathbb{R}}
\newcommand{\E}{\mathbb{E}}
\newcommand{\diagentry}[1]{\mathmakebox[1.8em]{#1}}
\newcommand{\xddots}{%
  \raise 4pt \hbox {.}
  \mkern 6mu
  \raise 1pt \hbox {.}
  \mkern 6mu
  \raise -2pt \hbox {.}
}








Following the framework for stochastic quasi-Newton methods (SQN) established in \citep{wang2017stochastic} for solving nonconvex stochastic optimization problems (see Section \ref{sec_15} in the appendix for this framework), we prove that, under fairly standard assumptions, {for} our K-BFGS(L) algorithm with skipping DD and exact inversion on $A_l$ 
(see Algorithm \ref{klbfgs_2} in Section 
\ref{sec_15}),
the number of iterations $N$ needed to obtain
{$\frac{1}{N}\sum_{k=1}^N\E[\| \mathbf{\nabla f}(\theta_k) \|^2]\le \epsilon$} is $N=O(\epsilon^{-\frac{1}{1-\beta}})$, for {step size} $\alpha_k$ chosen proportional to $k^{-\beta}$, where $\beta\in(0.5,1)$ is a constant.
Our proofs, which are delayed until Section 
\ref{sec_15},
make use of the following assumptions, the first two of which, were made in \citep{wang2017stochastic}.


\begin{assumption}
\label{assumption_6}
$f: \mathbb{R}^{n} \rightarrow \mathbb{R}$ is continuously differentiable.
{$f(\theta) \geq f^{low}> -\infty$}, 
for any
{$\theta \in \mathbb{R}^{n}$}. $\mathbf{\nabla f}$ is globally $L$-Lipschitz continuous;
namely for any $x, y \in \mathbb{R}^{n}$,
$
\| \mathbf{\nabla f}(x)- \mathbf{\nabla f}(y)\| \leq L\|x-y\|.
$
\end{assumption}


\begin{assumption}
\label{assumption_7}
For any iteration $k$, the stochastic gradient
{$\widehat{\mathbf{\nabla f}}_k = \widehat{\mathbf{\nabla f}}(\theta_k, \xi_k)$} satisfies:\\
a)
{$\mathbb{E}_{\xi_{k}} \left[ \widehat{\mathbf{\nabla f}}(\theta_{k}, \xi_{k}) \right] = \mathbf{\nabla f}(\theta_{k})$},
b)
{$\mathbb{E}_{\xi_{k}} \left[ \left\| \widehat{\mathbf{\nabla f}} (\theta_{k}, \xi_{k}) - \mathbf{\nabla f}(\theta_{k}) \right\|^{2} \right] \leq \sigma^{2}$},
where $\sigma>0$,
and $\xi_{k}, k=1,2, \ldots$ are independent samples
that are independent of
{$\left\{ \theta_{j} \right\}_{j=1}^{k}$}.
\end{assumption}


\begin{assumption}
\label{assumption_5}
The activation functions $\phi_l$ have bounded values: $\exists \varphi > 0$ s.t. $\forall l, \forall h, |\phi_l (h)| \leq \varphi$.
\end{assumption}

To use the convergence analysis in \citep{wang2017stochastic}, we need to show that the block-diagonal approximation of the inverse Hessian used in Algorithm {\ref{klbfgs_2}} satisfies the assumption that it is bounded above and below by positive-definite matrices. Given the Kronecker structure of our Hessian inverse approximation, it suffices to prove
{boundness} of both $H_a^l(k)$ and $H_g^l(k)$ for all iterations $k$. Making the additional assumption AS.\ref{assumption_5}, we are able to prove
Lemma \ref{Ha_bounds}, and hence Lemma \ref{lemma_2}, below.
Note that many popular activation functions satisfy AS.\ref{assumption_5}, such as sigmoid and tanh.

\begin{lemma}
\label{Ha_bounds}
Suppose that AS.\ref{assumption_5} holds. There exist two positive constants $\underline{\kappa}_a$, $\bar{\kappa}_a$ such that
$
    \underline{\kappa}_a I  \preceq  H_a^l(k)  \preceq  \bar{\kappa}_a I 
    ,\forall k, l.
$

\end{lemma}

\begin{lemma}
\label{lemma_3}

There exist two positive constants $\underline{\kappa}_g$ and $\bar{\kappa}_g$, such that
    $
    \underline{\kappa}_g I  \preceq  H_g^l(k)  \preceq  \bar{\kappa}_g I 
    ,\forall k, l.
    $
\end{lemma}

\begin{lemma}
\label{lemma_2}
Suppose that AS.\ref{assumption_5} holds. 
Let
{$\theta_{k+1} = \theta_k - \alpha_k H_k \widehat{\mathbf{\nabla f}}_k$} be the step taken in Algorithm \ref{klbfgs_2}. There exists two positive constants $\underline{\kappa}$, $\bar{\kappa}$ such that $\underline{\kappa} I \preceq H_k \preceq \bar{\kappa} I, \forall k$. 
\end{lemma}

Using Lemma \ref{lemma_2}, we can now apply 
Theorem 2.8 in \citep{wang2017stochastic} to prove the convergence of Algorithm \ref{klbfgs_2}:




\begin{theorem}
\label{thm_3}

Suppose that assumptions
AS.1-3
hold for
$\left\{ \theta_{k} \right\}$ generated by 
Algorithm \ref{klbfgs_2}
with mini-batch size $m_{k}=$ m for all $k $,
and $\alpha_{k}$ is chosen as
$\alpha_{k} = \frac{\underline{\kappa}}{L \bar{\kappa}^{2}} k^{-\beta}$,
with $\beta \in(0.5,1)$.
Then
$$\frac{1}{N} \sum_{k=1}^{N} \mathbb{E} \left[ \left\| \mathbf{\nabla f}(\theta_{k}) \right\|^{2} \right] \leq \frac{2 L\left(M_{f}-f^{l o w}\right) \bar{\kappa}^{2}}{\underline{\kappa}^{2}} N^{\beta-1}+\frac{\sigma^{2}}{(1-\beta) m}\left(N^{-\beta}-N^{-1}\right)$$
where $N$ denotes the iteration number {and $M_f > 0$ depends only on $f$}. Moreover, for a given $\epsilon \in(0,1),$ to guarantee that
$\frac{1}{N} \sum_{k=1}^{N} \mathbb{E}\left[ \left\| \mathbf{\nabla f}(\theta_{k}) \right\|^{2}\right]<\epsilon$, the number of iterations $N$ needed is at most
$O\left(\epsilon^{-\frac{1}{1-\beta}}\right)$.
\end{theorem}


{
Note: other theorems in \citep{wang2017stochastic}, namely Theorems 2.5 and 2.6, also apply here under our assumptions. 
}



\section{Experiments}
\label{sec_17}

Before we present some
{experimental} results, we address the use of moving averages, and the computational and storage requirements of the algorithms that we tested.


\textbf{Mini-batch and Moving Average. }
Clearly, using the whole dataset at each iteration is inefficient; hence, we use a mini-batch to estimate desired quantities. We use $\overline{X}$ to denote the averaged value of $X$ across the mini-batch for any quantity $X$. To incorporate information from the past as well as reducing the variability, we use an exponentially decaying moving average to estimate desired quantities
with decay parameter $\beta \in (0, 1)$:
\begin{enumerate}[topsep=0pt,itemsep=-1ex,partopsep=1ex,parsep=1ex,wide]
    \item To estimate 
    the gradient
    $\mathbb{E}_i [\mathbf{\nabla f}(i)]$,
    at each iteration, we update
    $\widehat{\mathbf{\nabla f}} = {\beta} \cdot \widehat{\mathbf{\nabla f}} + {(1-\beta)} \cdot \overline{\mathbf{\nabla f}}$.

    \item $H_a^l$: To estimate $A_l$, at each iteration we update
$
    \widehat{A_l} = {\beta} \cdot \widehat{A_l} + {(1-\beta)} \cdot \overline{\va_{l-1}\va_{l-1}^\top}.
    $
    {Note that although we compute $\vs_a^l$ as $H_a^l \cdot \overline{\va_{l-1}}$, we update $\widehat{A_l}$ with $\overline{\va_{l-1}\va_{l-1}^\top}$ (i.e. the average $\va_{l-1}(i) \va_{l-1}(i)^\top$ over the minibatch, not $ \overline{\va_{l-1}} \cdot ( \overline{ \va_{l-1}})^\top$).
    
    }
    

    \item $H^l_g$: BFGS "uses" momentum implicitly incorporated in the matrices $H^l_g$. To further stabilize the BFGS update, we also use a moving-averaged  $(\vs_g^l, \vy_g^l)$ (before damping); i.e., We update 
    $\vs_g^l = {\beta} \cdot \vs_g^l + {(1-\beta)} \cdot \left( \overline{\vh_l^+} - \overline{\vh_l} \right)$, and
    $\vy_g^l = {\beta} \cdot \vy_g^l + {(1-\beta)} \cdot \left( \overline{\vg_l^+} - \overline{\vg_l} \right)$.

\end{enumerate}

Finally, when computing
$\overline{\vh_l^+}$ and $\overline{\vg_l^+}$, we use the same mini-batch as was used to compute
$\overline{\vh_l}$ and $\overline{\vg_l}$.
This doubles the number of forward-backward passes at each iteration.

\textbf{Storage and Computational Complexity. }
Tables \ref{table_3} and \ref{table_4} compare the storage and computational requirements, respectively, for a layer with $d_i$ inputs and $d_o$ outputs for K-BFGS, K-BFGS(L), KFAC, and Adam/RMSprop. 
We denote the size of mini-batch by $m$, the number of $(\vs, \vy)$ pairs stored for L-BFGS by $p$, and the frequency of matrix inversion in KFAC by $T$. 
Besides the requirements listed in Table \ref{table_3}, all algorithms need storage for the parameters $W_l$ and the estimate of the gradient, $\widehat{\bf{\nabla f}_l}$, (i.e. $O(d_i d_o)$). Besides the work listed in Table \ref{table_4}, all algorithms also need to do a forward-backward pass to compute $\mathbf{\nabla f}_l$ as well as updating  $W_l$,  (i.e. $O(m d_i d_o)$). 
Also note that, even though we use big-$O$ notation in these tables, the constants for all of the terms in each of the rows are roughly at the same level and relatively small. 


In Table \ref{table_4}, for K-BFGS and K-BFGS(L), 
"Additional pass" refers to Line \ref{line_1} of Algorithm \ref{klbfgs};
under "Curvature", $O(m d_i^2)$ arises from "Hessian action" BFGS to update $H_a^l$ (see the algorithm at the end of Section \ref{sec_4}), $O(m d_o)$ arises from (\ref{eq_11}), $O(d_o^2)$ arises from updating $H_g^l$ (only for K-BFGS);
and "Step $\Delta W_l$" refers to (\ref{eq_1}). For KFAC, referring to Algorithm \ref{kfac} (in the appendix), "Additional pass" refers to Line \ref{line_2};
under "Curvature", $O(m d_i^2 + m d_o^2)$ refers to Line \ref{line_3}, and $O(\frac{1}{T} d_i^3 + \frac{1}{T} d_o^3)$ refers to Line \ref{line_4};
and "Step $\Delta W_l$" refers to Line \ref{line_5}.





From Table \ref{table_3}, we see that the Kronecker property enables K-BFGS and K-BFGS(L) (as well as KFAC) to have storage requirements comparable to those of first-order methods. Moreover, from Table \ref{table_4}, we see that K-BFGS and K-BFGS(L) require less computation per iteration than KFAC, since they only involve
{matrix} multiplications, whereas KFAC requires matrix inversions which depend cubically on both $d_i$ and $d_o$. The cost of matrix inversion in KFAC (and singular value decomposition in \cite{gupta2018shampoo}) is amortized by performing these operations only once every $T$ iterations; nonetheless, these amortized operations usually become much slower than
{matrix} multiplication as models scale up.


\begin{table}[ht]
  \caption{Storage}
  \label{table_3}
  \centering
  \begin{tabular}{l|cccllllll}
    \hline              
    Algorithm
    & $\nabla f_l \odot \nabla f_l$ & $A$ & $G$ & Total
    \\
    \hline
    K-BFGS
    & ---
    & $O(d_i^2)$
    & $O(d_o^2)$
    & $O(d_i^2 + d_o^2 + d_i d_o)$
    \\
    K-BFGS(L)
    & ---
    & $O(d_i^2)$
    & $O(p d_o)$
    & $O(d_i^2 + d_i d_o + p d_o)$
    \\
    KFAC
    &---
    & $O(d_i^2)$
    & $O(d_o^2)$
    & $O(d_i^2 + d_o^2 + d_i d_o)$
    \\
    Adam/RMSprop
    & $O(d_i d_o)$
    & ---
    & ---
    & $O(d_i d_o)$
    \\
    \hline
  \end{tabular}
\end{table}

\begin{table}[ht]
  \caption{Computation per iteration}
  \label{table_4}
  \centering
  \begin{tabular}{l|cccllllll}
    \hline
    Algorithm
    & Additional pass
    & Curvature
    & Step $\Delta W_l$
    \\
    \hline
    K-BFGS
    & $O(m d_i d_o)$
    & $O(m d_i^2 + m d_o + d_o^2)$
    & $O(d_i^2 d_o + d_o^2 d_i)$
    \\
    K-BFGS(L)
    & $O(m d_i d_o)$
    & $O(m d_i^2 + m d_o)$
    & $O(d_i^2 d_o + p d_i d_o)$
    \\
    KFAC
    & $O(m d_i d_o)$
    & $O(m d_i^2 + m d_o^2 + \frac{1}{T} d_i^3 + \frac{1}{T} d_o^3)$
    & $O(d_i^2 d_o + d_o^2 d_i)$
    \\
    Adam/RMSprop
    & ---
    & $O(d_i d_o)$
    & $O(d_i d_o)$
    \\
    \hline
  \end{tabular}
\end{table}

\textbf{Experimental Results. }
We tested K-BFGS and K-BFGS(L), as well as KFAC, Adam/RMSprop and SGD-m (SGD with momentum) on three autoencoder problems, namely, MNIST \cite{lecun1998gradient}, FACES, and CURVES, which are used in e.g. \cite{hinton2006reducing,martens2010deep,martens2015optimizing},
except that we replaced the sigmoid activation with ReLU. See Section \ref{sec_11} in the appendix for a complete description of these
problems and the competing algorithms. 


Since one can view
{Powell's damping with $B=I$} as LM damping, we write $\mu_2=\lambda_G$, where $\lambda_G$ denotes the LM damping parameter for $G_l$. We then define $\lambda = \lambda_A \lambda_G$ as the overall damping term of our QN approximation. 
To simplify matters, we chose $\lambda_A = \lambda_G = \sqrt{\lambda}$, so that we needed to tune only one hyper-parameter (HP) $\lambda$.

\begin{figure}[H]
  \centering
  \begin{subfigure}[b]{0.49\textwidth}
    \includegraphics[width=\textwidth, height=7cm]{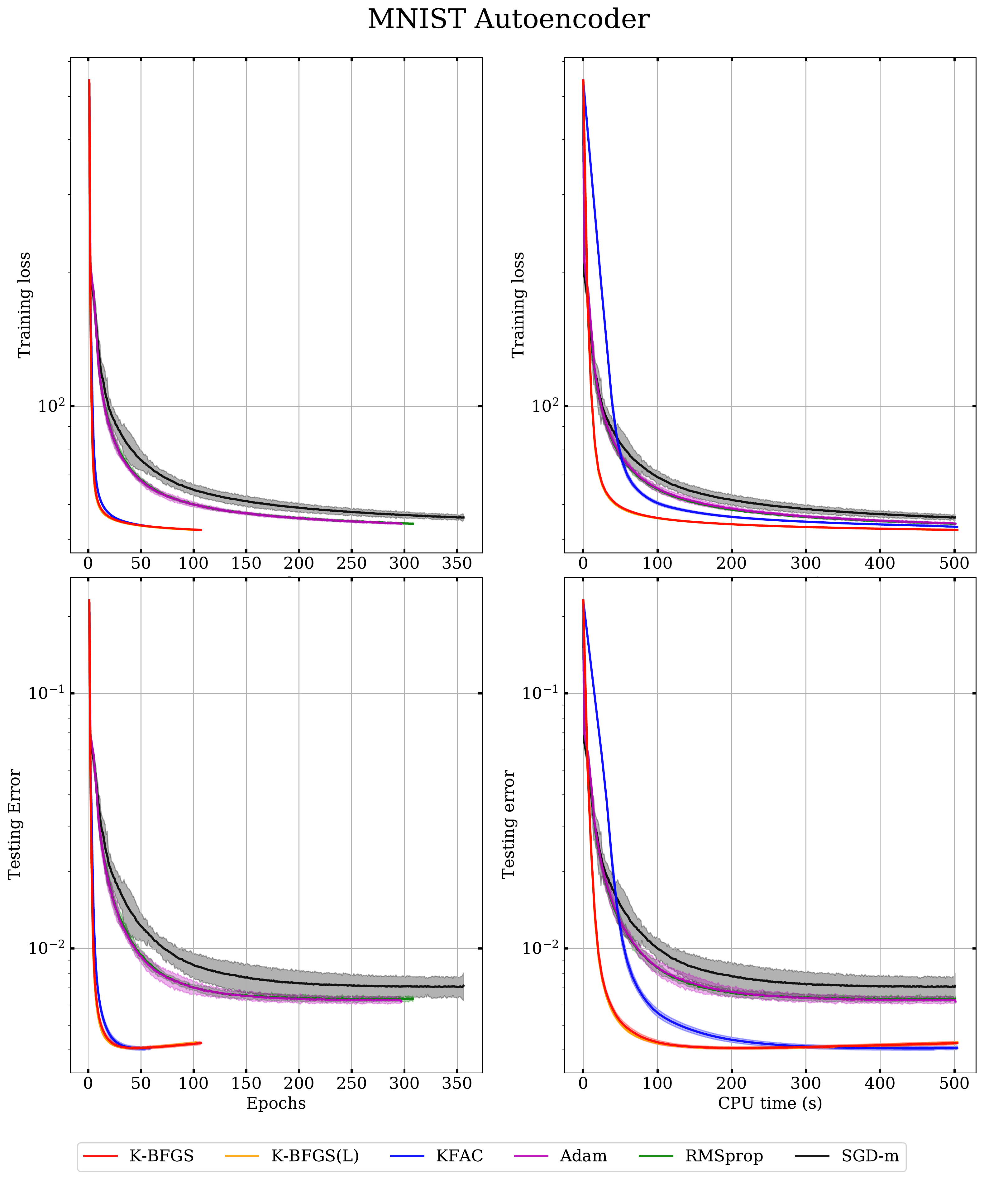}
    \caption{MNIST}
  \end{subfigure}
  \begin{subfigure}[b]{0.49\textwidth}
    \includegraphics[width=\textwidth, height=7cm]{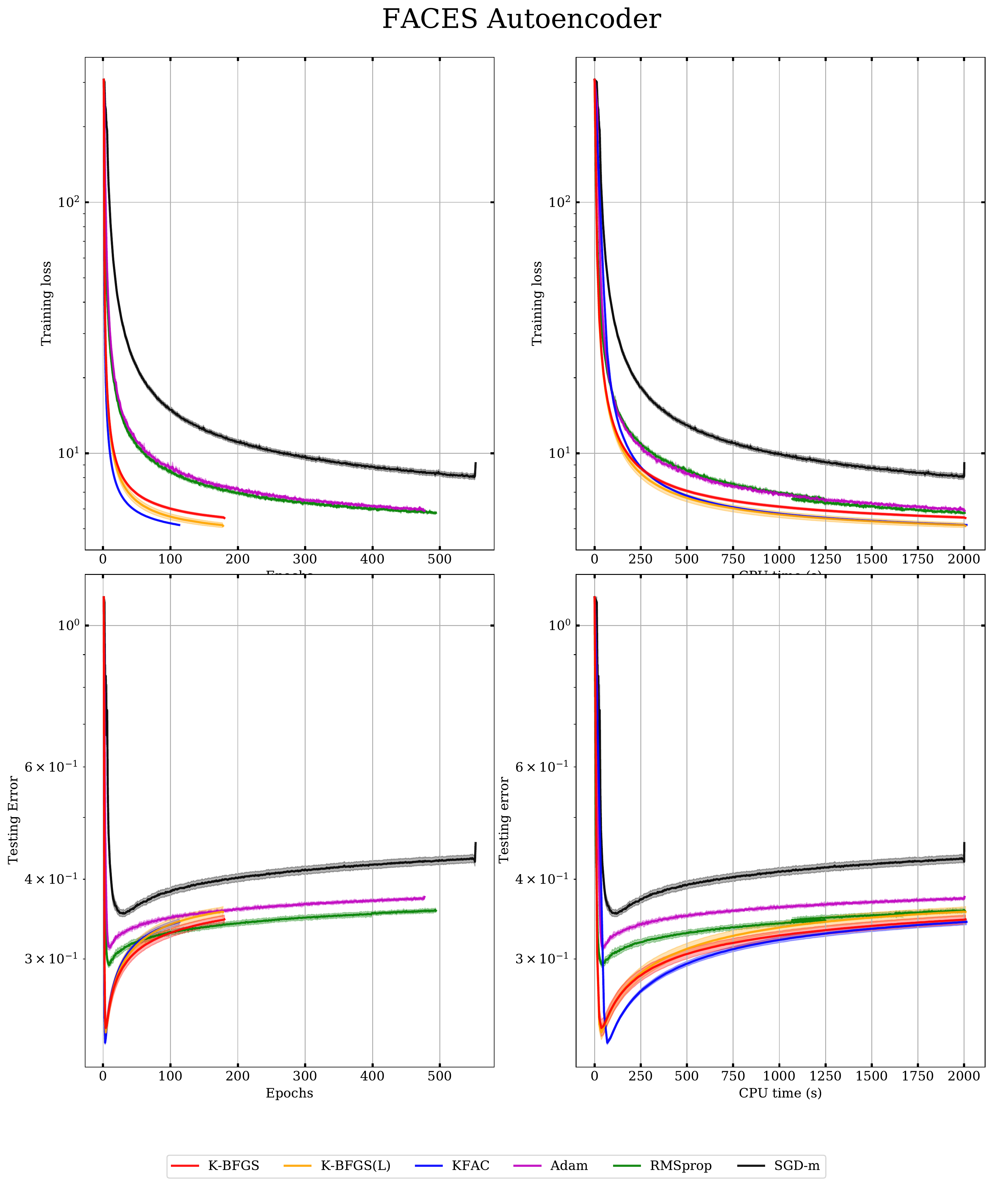}
    \caption{FACES}
  \end{subfigure}
  \begin{subfigure}[b]{0.49\textwidth}
    \includegraphics[width=\textwidth, height=7cm]{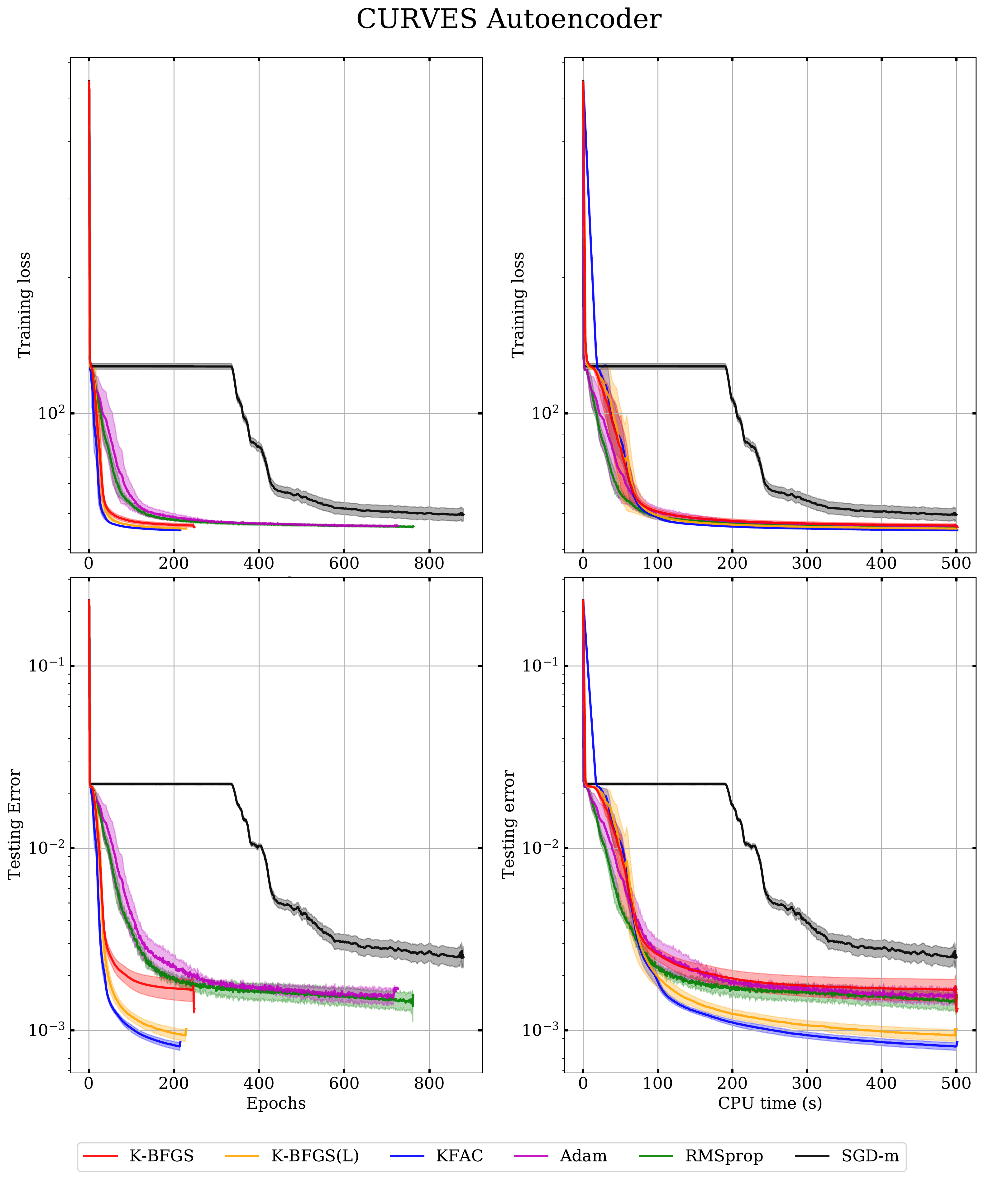}
    \caption{CURVES}
  \end{subfigure}

  \caption{
  Comparison between algorithms on (a) MNIST, (b) FACES, (c) CURVES. For each problem, the upper (lower) rower depicts training loss (testing (mean square) error), whereas
  the left (right) column depicts training/test progress versus epoch (CPU time), respectively.
After each epoch, the training loss/testing error from the whole training/testing set is reported
(the time for computing this loss is not included in the plots).
For each problem, algorithms are terminated after the same amount of CPU time.
  }
  \label{fig_5}
\end{figure}

To obtain the results in Figure \ref{fig_5}, we first did a grid-search on (learning rate, damping) pairs for all algorithms (except for SGD-m, whose grid-search was only on learning rate), where damping refers to 
$\lambda$ for K-BFGS/K-BFGS(L)/KFAC, and $\epsilon$ for RMSprop/Adam. We then selected the best (learning rate, damping) pairs with the lowest training loss encountered.
The range for the grid-search and the best 
HP values (as well as other fixed 
HP values) are listed in Section \ref{sec_11} in the appendix. Using the best HP values that we found, we then made 20 runs employing different random seeds, and plotted the mean value of the 20 runs as the solid line and the standard deviation as the shaded area.\footnote{
Results were obtained on a machine with 8 x Intel(R) Xeon(R) CPU @ 2.30GHz and 1 x NVIDIA Tesla P100.
Code is available at \url{https://github.com/renyiryry/kbfgs_neurips2020_public}.
}

{
From the training loss plots in Figure \ref{fig_5}, our algorithms clearly outperformed the first-order methods, except for RMSprop/Adam on CURVES, with respect to CPU time, and performed comparably to KFAC in terms of both CPU time and number of epochs taken. The testing error plots in  Figure \ref{fig_5}  show that our K-BFGS(L) method and KFAC behave very similarly and substantially outperform all of the first-order methods in terms of both of these measures. This suggests that our algorithms not only optimize well, but also generalize well.
}

To further demonstrate the robustness of our algorithms, we examined the loss under various HP settings,  
which showed that our algorithms are stable under a fairly wide range for the HPs (see Section \ref{sec_16} in the appendix). 

We also repeated our experiments using mini-batches of size 100 for all algorithms (see Figures \ref{fig_2}, \ref{fig_3}, and \ref{fig_4} in the appendix, where HPs are optimally tuned for batch sizes of 100) and our proposed methods continue to demonstrate advantageous performance, both in training and testing. For these experiments, we did not experiment with 20 random seeds.
These results show that our approach works as well for relatively small mini-batch sizes of 100, as those of size 1000, which are less noisy, and hence is robust in the stochastic setting employed to train DNNs.

Compared with other methods mentioned in this paper, our K-BFGS and K-BFGS(L) have the extra advantage of being able to double the size of minibatch for computing the stochastic gradient with almost no extra cost, which might be of particular interest in a highly stochastic setting. See Section \ref{sec_2} in the appendix for more discussion on this and some preliminary experimental results.


\section{Conclusion}

We proposed Kronecker-factored QN methods, namely, K-BFGS and K-BFGS(L), for training multi-layer feed-forward neural network models, that use layer-wise second-order information and require modest memory and computational resources. Experimental results indicate that our methods outperform or perform comparably to the state-of-the-art first-order and second-order methods. Our methods can also be extended to convolutional and recurrent NNs.



\clearpage

\small

\section*{Broader Impact}



{The research presented in this paper provides a new method for training DNNs that our experimental testing has shown to be more efficient in several cases than current state-of-the-art optimization methods for this task.  Consequently, because of the wide use of DNNs in machine learning, this should help save a substantial amount of energy.
Our new algorithms simply attempt to minimize the loss function that are given to it and the computations that it performs are all transparent. In machine learning, DNNs can be trained to address many types of problems, some of which should lead to positive societal outcomes, such as ones in medicine (e.g., diagnostics and drug effectiveness), autonomous vehicle development, voice recognition and climate change. Of course, optimization algorithms for training DNNs can be used to train models that may have negative consequences, such as those intended to develop psychological profiles,
invade privacy and justify biases.  The misuse of any efficient optimization algorithm for machine learning, and in particular our algorithms, is beyond the control of the work presented here.

}

\begin{ack}

DG and YR were supported in part by NSF Grant IIS-1838061. DG acknowledges the computational support 
provided by Google Cloud Platform Education Grant, Q81G-U4X3-5AG5-F5CG.

\end{ack}

\medskip

\bibliography{references.bib}

\appendix
\newpage

\section{Pseudocode for K-BFGS/K-BFGS(L)}


Algorithm \ref{algo_3} gives pseudocode for K-BFGS/K-BFGS(L), which is implemented in the experiments. For details see Sections \ref{sec_1}, \ref{sec_4}, and Section \ref{sec_5} in the Appendix. 

\begin{algorithm}
    \caption{
    Pseudocode for K-BFGS / K-BFGS(L)
    }
    \label{algo_3}
    \begin{algorithmic}[1]

    \REQUIRE
    Given initial weights $\theta= \left[\operatorname{vec}\left(W_{1}\right)^{\top},  \ldots, \operatorname{vec}\left(W_{L}\right)^{\top}\right]^{\top}$, batch size $m$, learning rate
    $\alpha$, {damping value $\lambda$}, and for K-BFGS(L), the number of $(\vs, \vy)$ pairs $p$ that are
    stored and used to compute $H_g^l$ at each iteration
    
    \STATE $\mu_1 = 0.2$, $\beta = 0.9$
    \COMMENT{set default hyper-parameter values}
    
    \STATE $\lambda_A = \lambda_G = \sqrt{\lambda}$
    \COMMENT{split the damping into $A$ and $G$}
    
    
    \STATE $\widehat{\mathbf{\nabla f}}_l = 0$,
    $A_l = \mathbb{E}_i \left[ \va_{l-1}(i) \va_{l-1}(i)^\top \right]$ {by forward pass}, $H_a^l = (A_l + \lambda_A I_A)^{-1}$, $H_g^l = I$ ($l = 1,...,L$)
    \COMMENT{Initialization}
    

    \FOR {$k=1,2,\ldots$}
        \STATE Sample mini-batch of size {$m$}: $M_k = \{\xi_{k,i}, i =1, \dots, m\}$
        
        \STATE Perform a forward-backward pass over the current mini-batch $M_k$ to compute $\overline{\mathbf{\nabla f}}_l$, $\va_l$, $\vh_l$, and $\vg_l$ ($l = 1, \ldots, L$) (see Algorithm \ref{feedforward})
        
        \FOR {$l=1, \ldots, L$}
            
            \STATE 
            $\widehat{\mathbf{\nabla f}}_l = \beta \widehat{\mathbf{\nabla f}}_l + (1-\beta) \overline{\mathbf{\nabla f}}_l$
            
            
            
            
            
            \STATE $p_l = H_g^l \widehat{\mathbf{\nabla f}}_l H_a^l$
            
            \STATE
        
            \COMMENT{In K-BFGS(L),
            when computing $ H^l_g \left(\widehat{\mathbf{\nabla f}}_l H^l_a  \right)$, L-BFGS is initialized with an identity matrix}

            \STATE $W_l = W_l - \alpha \cdot p_{l}$
        \ENDFOR
        
        \STATE Perform another forward-backward pass over $M_k$ to compute $\vh_l^+$, $\vg_l^+$ ($l = 1, \ldots, L$)
        
        \FOR {$l = 1,...,L$}
        
        \STATE 
        \COMMENT{Use damped BFGS or L-BFGS to update $H_g^l$ (see Section \ref{sec_1})}
        
        

        \STATE
        $\vs_g^l = {\beta} \cdot \vs_g^l + {(1-\beta)} \cdot \left( \overline{\vh_l^+} - \overline{\vh_l} \right)$,
        {$\vy_g^l = {\beta} \cdot \vy_g^l + {(1-\beta)} \cdot \left( \overline{\vg_l^+} - \overline{\vg_l} \right)$}

        \STATE $(\tilde{\vs}_g^l, {\tilde{\vy}_g^l}) = \text{DD}({\vs}_g^l, {{\vy}_g^l})$ with $H = H_g^l$, $\mu_1 = \mu_1$, $\mu_2 = \lambda_G$
        \COMMENT{See Algorithm \ref{algo_2}}  
        
        \STATE Use BFGS or L-BFGS with $(\tilde{\vs}_g^l, {\tilde{\vy}_g^l})$ to update $H_g^l$
        
        

    

    

    
        
        \STATE 
        \COMMENT{Use Hessian-action BFGS
        to update $H_a^l$ (see Section \ref{sec_4})}

        \STATE
        $A_l = \beta \cdot A_l + (1-\beta) \cdot \overline{\va_{l-1} \va_{l-1}^\top}$ 
        
        \STATE $A_l^{\text{LM}} = A_l + \lambda_A I_A$
    
        \STATE
        $\vs_a^l = H_a^l \cdot \overline{\va_{l-1}}$,
        $\vy_a^l = A_l^{\text{LM}} \vs_a^l$
        
        \STATE Use BFGS with $(\vs_a^l, \vy_a^l)$ to update $H_a^l$

        \ENDFOR

        \ENDFOR  
    
    \end{algorithmic}
\end{algorithm}

\section{Convergence: Proofs of Lemmas 1-3 and Theorem 2}
\label{sec_15}



\begin{algorithm}
    \caption{
    K-BFGS(L) with DD-skip and exact inversion of  $A_l^{\text{LM}}$
    }
    \label{klbfgs_2}
    \begin{algorithmic}[1]

    \REQUIRE
    Given initial weights $\theta= \left[\operatorname{vec}\left(W_{1}\right)^{\top},  \ldots, \operatorname{vec}\left(W_{L}\right)^{\top}\right]^{\top}$, batch size $m$, learning rate
    $\{ \alpha_k \}$, damping value $\lambda$, and the number of $(\vs, \vy)$ pairs $p$
    that are stored and used to compute $H_g^l$ at each iteration
    
    \STATE $\mu_1 = 0.2$, $\beta = 0.9$
    \COMMENT{set default hyper-parameter values}
    
    \STATE $\lambda_A = \lambda_G = \sqrt{\lambda}$
    \COMMENT{split the damping into $A$ and $G$}
    
    \STATE
    $A_l(0) = \mathbb{E}_i \left[ \va_{l-1}(i) \va_{l-1}(i)^\top \right]$ by forward pass, $H_a^l(0) = (A_l(0) + \lambda_A I_A)^{-1}$, $H_g^l(0) = I$ ($l = 1,...,L$)
    \COMMENT{Initialization}
    

    \FOR {$k=1,2,\ldots$}
        \STATE Sample mini-batch of size {$m$}: $M_k = \{\xi_{k,i}, i =1, \dots, m\}$
        
        \STATE Perform a forward-backward pass over the current mini-batch $M_k$ to compute $\overline{\mathbf{\nabla f}}_l$, $\va_l$, $\vh_l$, and $\vg_l$ ($l = 1, \ldots, L$) (see Algorithm \ref{feedforward})
        
        \FOR {$l=1, \ldots, L$}
            

            \STATE 
            \label{line_13}
            $p_l = H_g^l(k-1) \widehat{\mathbf{\nabla f}}_l H_a^l(k-1)$, where $\widehat{\mathbf{\nabla f}}_l = \overline{\mathbf{\nabla f}}_l$
            
            \STATE
            \COMMENT{When computing $ H^l_g \left(\widehat{\mathbf{\nabla f}}_l H^l_a  \right)$, L-BFGS is initialized with an identity matrix}

            \STATE $W_l = W_l - \alpha_k \cdot p_{l}$
        \ENDFOR
        
        \STATE Perform another forward-backward pass over $M_k$ to compute $\vh_l^+$, $\vg_l^+$ ($l = 1, \ldots, L$)
        
        \FOR {$l = 1,...,L$}
        
        \STATE 
        \COMMENT{Use damped L-BFGS {with skip} to update $H_g^l$ (see Section \ref{sec_1})}
        
        

        \STATE
        $\vs_g^l = {\beta} \cdot \vs_g^l + {(1-\beta)} \cdot \left( \overline{\vh_l^+} - \overline{\vh_l} \right)$,
        {$\vy_g^l = {\beta} \cdot \vy_g^l + {(1-\beta)} \cdot \left( \overline{\vg_l^+} - \overline{\vg_l} \right)$}

        \STATE $(\tilde{\vs}_g^l, {\tilde{\vy}_g^l}) = \text{DD}({\vs}_g^l, {{\vy}_g^l})$ with $H = H_g^l(k-1)$, $\mu_1 = \mu_1$, $\mu_2 = \lambda_G$
        \COMMENT{See Algorithm \ref{algo_2}}  
        
        
        {
        \IF {$(\tilde{\vs}_g^l)^\top \tilde{\vy}_g^l \ge \mu_1 (\tilde{\vy}_g^l)^\top H_g^l \tilde{\vy}_g^l$}
        \label{line_11}
            
            \STATE Use L-BFGS with $(\tilde{\vs}_g^l, {\tilde{\vy}_g^l})$ to update $H_g^l(k)$
            
        \ENDIF
        }

        \STATE 
        \COMMENT{Use {exact inversion to compute} $H_a^l$}

        \STATE
        $A_l(k) = \beta \cdot A_l(k-1) + (1-\beta) \cdot \overline{\va_{l-1} \va_{l-1}^\top}$ 
        
        \STATE $A_l^{\text{LM}}(k) = A_l(k) + \lambda_A I_A$
    
        
        
        \STATE
        \label{line_12}
        {
        $H_a^l(k) = \left( A_l^{\text{LM}}(k) \right)^{-1}$
        }

        \ENDFOR

        \ENDFOR  
    
    \end{algorithmic}
\end{algorithm}

In this section, we prove the convergence of Algorithm \ref{klbfgs_2}, a variant of K-BFGS(L).
Algorithm \ref{klbfgs_2} is very similar to our actual implementation of K-BFGS(L) (i.e. Algorithm \ref{algo_3}), except that
\begin{itemize}
    \item we skip updating $H_g^l$ if $(\tilde{\vs}_g^l)^\top \tilde{\vy}_g^l < \mu_1 (\tilde{\vy}_g^l)^\top H_g^l \tilde{\vy}_g^l$ (see Line \ref{line_11});
    
    \item we set $H_a^l$ to the exact inverse of $A_l^{\text{LM}}$ (see Line \ref{line_12});
    
    \item we use decreasing step sizes $\{ \alpha_k \}$ as specified in Theorem \ref{thm_3};
    
    \item we use the mini-batch gradient instead of the momentum gradient (see Line \ref{line_13}). 
\end{itemize}

To accomplish this, we prove Lemmas 1-3, which in addition to Assumptions AS.1-2, ensure that all of the assumptions in Theorem 2.8 in \citep{wang2017stochastic} are satisfied, and hence that the generic stochastic quasi-Newton (SQN) method, i.e. Algorithm \ref{sqn}, below converges. 
Specifically, Theorem 2.8 in \citep{wang2017stochastic} requires, in addition to Assumptions AS.1-2, the assumption
\begin{assumption}
\label{assumption_8}
There exist two positive constants $\underline{\kappa}, \bar{\kappa}$, such that
$\underline{\kappa} I \preceq H_{k} \preceq \bar{\kappa} I, \forall k$; 
for any $k \geq 2,$ the random variable $H_{k}$ depends only on $\xi_{[k-1]}$.
\end{assumption}

\begin{algorithm}
    \caption{SQN method for nonconvex stochastic optimization.}
    \label{sqn}
    \begin{algorithmic}[1]

    \REQUIRE
    Given
    {$\theta_{1} \in \mathbb{R}^{n}$}, 
    batch sizes $\left\{m_{k}\right\}_{k \geq 1}$, and step sizes $\left\{\alpha_{k}\right\}_{k \geq 1}$

    \FOR {$k=1,2,\ldots$}
        \STATE Calculate
        $\widehat{\mathbf{\nabla f}}_{k}=\frac{1}{m_{k}} \sum_{i=1}^{m_{k}} \mathbf{\nabla f}(\theta_{k}, \xi_{k, i})$

        \STATE Generate a positive definite Hessian inverse approximation $H_{k}$

        \STATE Calculate
        {$\theta_{k+1} = \theta_{k} - \alpha_{k} H_{k} \widehat{\mathbf{\nabla f}}_{k}$} 
        
        \ENDFOR  
    
    \end{algorithmic}
\end{algorithm}

In the following proofs, $\|\cdot\|$ denotes the 2-norm for vectors, and the spectral norm for matrices.


{\bf Proof of Lemma \ref{Ha_bounds}}:
\begin{proof}

Because $A_l^{\text{LM}}(k) \succeq \lambda_A I_A$, we have that $H_a^l(k) \preceq \bar{\kappa}_a I_A$, where $\bar{\kappa}_a = \frac{1}{\lambda_A}$. 

On the other hand, for any $\vx \in \R^{d_l}$, by Cauchy-Schwarz, $\langle {\va}_{l-1}(i), \vx  \rangle^2\; \leq \|\vx\|^2 \| {\va}_{l-1}(i) \|^2
\leq \|\vx\|^2 (1 + \varphi^2 d_l)$. Hence, $\left\| \overline{\va_{l-1} \va_{l-1}^\top} \right\| \leq 1 + \varphi^2 d_l$; similarly, $\| A_l(0) \| \leq 1 + \varphi^2 d_l$. Because $\| A_l(k) \| \leq \beta \| A_l(k-1) \| + (1-\beta) \left\| \overline{\va_{l-1} \va_{l-1}^\top} \right\|$, by induction, $\| A_l(k) \| \leq 1 + \varphi^2 d_l$ for any $k$ and $l$. Thus, $\| A_l^{\text{LM}}(k) \| \leq 1 + \varphi^2 d_l + \lambda_A$. Hence, $H_a^l(k) \succeq \underline{\kappa}_a I_A$, where $\underline{\kappa}_a = \frac{1}{1 + \varphi^2 d_l + \lambda_A}$.

\end{proof}

{\bf Proof of Lemma \ref{lemma_3}:}

\begin{proof}




To simplify notation, we omit the subscript $g$, superscript $l$ and the iteration index $k$ in the proof.
{Hence, our goal is to prove $\underline{\kappa}_g I \preceq H = H_g^l(k) \preceq \bar{\kappa}_g I$, 
for any $l$ and $k$.}
Let
{$(\vs_i, \vy_i)$ ($i = 1, ..., p$)}
denote the pairs used in an L-BFGS computation of 
$H$. Since $(\vs_i, \vy_i)$ was 
{\bf not skipped},
{$\frac{\vy_i^\top \bar{H}^{(i)} \vy_i}{\vs_i^\top \vy_i} \leq \frac{1}{\mu_1}$}, where
$\bar{H}^{(i)}$ denotes the matrix $H^l_g$ used at the iteration in which
{$\vs_i$ and $\vy_i$ were} computed. Note that this is not the matrix $H_i$ used in the recursive computation of $H$ at the current iterate $\theta_{k}$.

Given an initial estimate
{$H_0 =B_0^{-1}= I$} of $(G^l_g(\theta_k))^{-1}$,
the L-BFGS method updates
$H_{i}$ recursively as
{
\begin{align}
    H_{i}
    =\left(I-\rho_{i} \vs_i \vy_i^{\top} \right) H_{i-1} \left(I-\rho_{i} \vy_{i} \vs_{i}^{\top}\right)+\rho_{i} \vs_{i} \vs_{i}^{\top}, \quad i=1, \ldots, p,
\label{eq_23}
\end{align}
}
where
{$\rho_i = (\vs_i^\top \vy_i)^{-1}$}, and equivalently, 
{
\[
    B_{i}
    = B_{i-1}+\frac{{\vy}_i {\vy}_i^{\top}}{\vs_i^{\top} {\vy}_i}-\frac{B_{i-1} \vs_i \vs_i^{\top} B_{i-1}}{\vs_i^{\top} B_{i-1} \vs_i}, \quad i=1, \ldots, p,
\]
}

where $B_i = H_i^{-1}$.
Since we use DD with skipping, we have that 
{$\frac{\vs_i^\top \vs_i}{\vs_i^\top \vy_i} \leq \frac{1}{\mu_2}$}
and
{$\frac{\vy_i^\top \bar{H}^{(i)} \vy_i}{\vs_i^\top \vy_i} \leq \frac{1}{\mu_1}$}.
Note that we don't have
{$\frac{\vy_i^\top H_{i-1} \vy_i}{\vs_i^\top \vy_i} \leq \frac{1}{\mu_1}$}, so we cannot direct apply (\ref{eq_22}).
Hence, by (\ref{eq_21}), we have that
{$||B_i|| \le ||B_{i-1}|| \left( 1 + \frac{1}{\mu_1} \right)$}. Hence,
{$||B|| = ||B_p|| \le ||B_0|| \left( 1 + \frac{1}{\mu_1} \right)^p = \left( 1 + \frac{1}{\mu_1} \right)^p$}. Thus,
{$B \preceq \left( 1 + \frac{1}{\mu_1} \right)^p I$},
{$H \succeq \left( 1 + \frac{1}{\mu_1} \right)^{-p} I := \underline{\kappa}_g I$}.


On the other hand, since $\underline{\kappa}_g$ is a uniform lower bound for $H_g^l(k)$ for any $k$ and $l$, $\bar{H}^{(i)} \succeq \underline{\kappa}_g I$. Thus,
{
\begin{align*}
    \frac{1}{\mu_1}
    \geq \frac{\vy_i^\top \bar{H}^{(i)} \vy_i}{\vs_i^\top \vy_i}
    \geq \underline{\kappa}_g \frac{\vy_i^\top \vy_i}{\vs_i^\top \vy_i}
    \Rightarrow
    \frac{\vy_i^\top \vy_i}{\vs_i^\top \vy_i} \leq \frac{1}{\mu_1 \underline{\kappa}_g}.
\end{align*}
}

Hence,
using the fact that $|| u v^\top || = ||u|| \cdot ||v||$ for any vectors $u,v$,
$
    ||\rho_i \vs_i \vs_i^\top||
    = \rho_i || \vs_i || || \vs_i||
    = \frac{\vs_i^\top \vs_i}{\vs_i^\top \vy_i}
    \leq \frac{1}{\mu_2},
$
\begin{align*}
    ||H_{i}||
    & = || \left(I-\rho_i \vs_i \vy_i^{\top}\right) H_{i-1} \left(I-\rho_i \vy_i \vs_i^{\top}\right)+\rho_i \vs_i \vs_i^{\top} ||
    \\
    & = || H_{i-1} + \rho_i^2 (\vy_i^\top H_{i-1} \vy_i) \vs_i \vs_i^\top - \rho_i \vs_i \vy_i^\top H_{i-1} - \rho_i H_{i-1} \vy_i \vs_i^\top + \rho_i \vs_i \vs_i^{\top} ||
    \\
    & \leq || H_{i-1} || + || \rho_i^2 (\vy_i^\top H_{i-1} \vy_i) \vs_i \vs_i^\top || + || \rho_i \vs_i \vy_i^\top H_{i-1} || + || \rho_i H_{i-1} \vy_i \vs_i^\top || + || \rho_i \vs_i \vs_i^{\top} ||
    \\
    & \leq || H_{i-1} || + || H_{i-1} || \cdot || \rho_i^2 (\vy_i^\top \vy_i) \vs_i \vs_i^\top || + 2 \rho_i || \vs_i || \cdot || \vy_i^\top H_{i-1} || 
    + \frac{1}{\mu_2}
    \\
    & \leq || H_{i-1} || + || H_{i-1} || \cdot \frac{1}{\mu_1 \underline{\kappa}_g} \frac{1}{\mu_2} + 2 \rho_i || \vs_i || \cdot || \vy_i^\top || \cdot || H_{i-1} || + \frac{1}{\mu_2}
    \\
    & \leq || H_{i-1} || \left( 1 + \frac{1}{\mu_1 \underline{\kappa}_g} \frac{1}{\mu_2} + 2 \frac{1}{\sqrt{\mu_1 \mu_2 \underline{\kappa}_g}} \right) + \frac{1}{\mu_2}
    \\
    & = \hat{\mu}|| H_{i-1} || + \frac{1}{\mu_2}, \quad {\rm where} \quad \hat{\mu} = \left( 1 + \frac{1}{\sqrt{\mu_1 \mu_2 \underline{\kappa}_g}} \right)^2. 
\end{align*}
From the fact that $H_0 =I$, and induction, we have that $||H|| \leq \hat{\mu}^p + \frac{\hat{\mu}^p - 1}{ \hat{\mu} - 1}\frac{1}{\mu_2} \equiv \bar{\kappa}_g$. 

\end{proof}

{\bf Proof of Lemma \ref{lemma_2}:}


\begin{proof}
By Lemma \ref{Ha_bounds}, \ref{lemma_3} and the fact that
$H_k = \text{diag}\{ H_a^1(k-1) \otimes H_g^1(k-1), ..., H_a^L(k-1) \otimes H_g^L(k-1) \}$. 
\end{proof}

{\bf Proof of Theorem \ref{thm_3}:}
\begin{proof}

To show that Algorithm \ref{klbfgs_2} lies in the framework of Algorithm \ref{sqn}, it suffices to show that $H_k$ generated by Algorithm \ref{klbfgs_2} is positive definite, which is true since
{$H_k = \text{diag}\{ H_a^1(k-1) \otimes H_g^1(k-1), ..., H_a^L(k-1) \otimes H_g^L(k-1) \}$} and $H_a^l(k)$ and $H_g^l(k)$ are positive definite for all $k$ and $l$. 
Then 
by Lemma \ref{lemma_2},
and the fact that $H_k$ depends on $H_a^l(k-1)$ and $H_g^l(k-1)$, and $H_a^l(k-1)$ and $H_g^l(k-1)$ does not depend on random samplings in the $k$th iteration,
AS.\ref{assumption_8} holds.
Hence, Theorem 2.8 of \cite{wang2017stochastic} applies to Algorithm \ref{klbfgs_2}, proving Theorem \ref{thm_3}.
\end{proof}

\section{Powell's Damped BFGS Updating}
\label{sec_5}



For BFGS and L-BFGS, one needs $\vy^\top \vs >0$. However, when used to update $H_g^l$, there is no guarantee that $(\vy^l_g)^\top \vs^l_g >0$ for any layer $l = 1, \ldots,L$. 
In deterministic optimization, positive definiteness of the QN Hessian approximation $B$ (or its inverse) is maintained by performing an inexact line search that ensures that $\vs^T \vy >0$, which is always possible as long as the function being minimized is bounded below.
However, this would be expensive to do for DNN. Thus, we propose the following heuristic based on Powell's damped-BFGS approach \cite{powell1978algorithms}.

\textbf{Powell's Damping on \texorpdfstring{$B$}{TEXT}. }
Powell's damping on $B$, proposed in \cite{powell1978algorithms}, replaces $\vy$ in
the BFGS update, by
$
\tilde{\vy} = \theta \vy + (1-\theta) B \vs
$, 
where
\begin{align*}
\theta = 
\begin{cases}
\frac{(1-\mu) \vs^\top B \vs}{\vs^\top B \vs - \vs^\top \vy}, & \text{if } \vs^\top \vy < \mu \vs^\top B \vs,
\\
1, & \text{otherwise.}
\end{cases}
\end{align*}
It is easy to verify that
    {$\vs^\top \tilde{\vy} \ge \mu \vs^\top B \vs$}.

\textbf{Powell's Damping on \texorpdfstring{$H$}{TEXT}. }
In Powell's damping on $H$ (see e.g. \cite{badreddine2014sequential}),
$
\tilde{\vs} = \theta \vs + (1-\theta) H \vy
$
replaces $\vs$, where
\begin{align*}
\theta = 
\begin{cases}
\frac{(1-\mu_1) \vy^\top H \vy}{\vy^\top H \vy - \vs^\top \vy}, 
& \text{if } \vs^\top \vy < \mu_1 \vy^\top H \vy,
\\
1, 
& \text{otherwise.}
\end{cases}
\end{align*}
This is used in lines \ref{line_7} and \ref{line_8} of the DD (Algorithm  \ref{algo_2}). It is also easy to verify that 
{$\tilde{\vs}^\top \vy \ge \mu_1 \vy^\top H \vy$}.

    
    

\textbf{Powell's Damping with $B = I$. }
Powell's damping on $B$ is not suitable for our algorithms because we do not keep track of $B$. Moreover, it does not provide a simple bound on
{$\frac{\tilde{\vs}^\top \tilde{\vs}}{\tilde{\vs}^\top \tilde{\vy}}$} that is independent of $\| B \|$. Therefore, we use Powell's damping with $B = I$, in lines \ref{line_9} and \ref{line_10} of the DD (Algorithm \ref{algo_2}). 
%
%
%
It is easy to verify that it ensures that
{$\tilde{\vs}^\top \tilde\vy \geq \mu_2 \tilde{\vs}^\top \tilde{\vs}$}.


Powell's damping with $B=I$ can be interpreted as adding an 
{Levenberg-Marquardt (LM)} damping term to $B$. Note that an 
LM damping term {$\mu_2$} would lead to
{$B \succeq \mu_2 I$}. Then, the secant condition
{$\tilde{\vy} = B \tilde{\vs}$} implies
$$
\tilde{\vy}^\top \tilde{\vs} = \tilde{\vs}^\top B \tilde{\vs} \ge \mu_2 \tilde{\vs}^\top \tilde{\vs},
$$
which is the same inequality as we get using Powell's damping with $B = I$. 
Note that although the {$\mu_2$} parameter in Powell's damping with $B=I$ can be interpreted as an LM damping, we recommend setting the value of $\mu_2$ within $(0, 1]$ so that $\tilde{\vy}$ is a convex combination of $\vy$ and
{$\tilde{\vs}$}.
In all of our experimental tests, we found that the best value for the hyperparameter $\lambda$ for both K-BFGS and K-BFGS(L) 
was less than or equal to $1$, and hence that $\mu_2 = \lambda_G = \sqrt{\lambda}$ was in the interval $(0,1]$.



\subsection{Double Damping {(DD)}}

Our double damping (Algorithm \ref{algo_2}) is a two-step damping procedure, where the first step (i.e. Powell's damping on $H$) can be viewed as an interpolation between the current curvature and the previous ones, and the second step (i.e.
{Powell's damping with $B = I$}) can be viewed as an
{LM} damping. 

\begin{figure}[ht]
  \centering
  \begin{minipage}{.5\textwidth}
    \includegraphics[width=\textwidth]{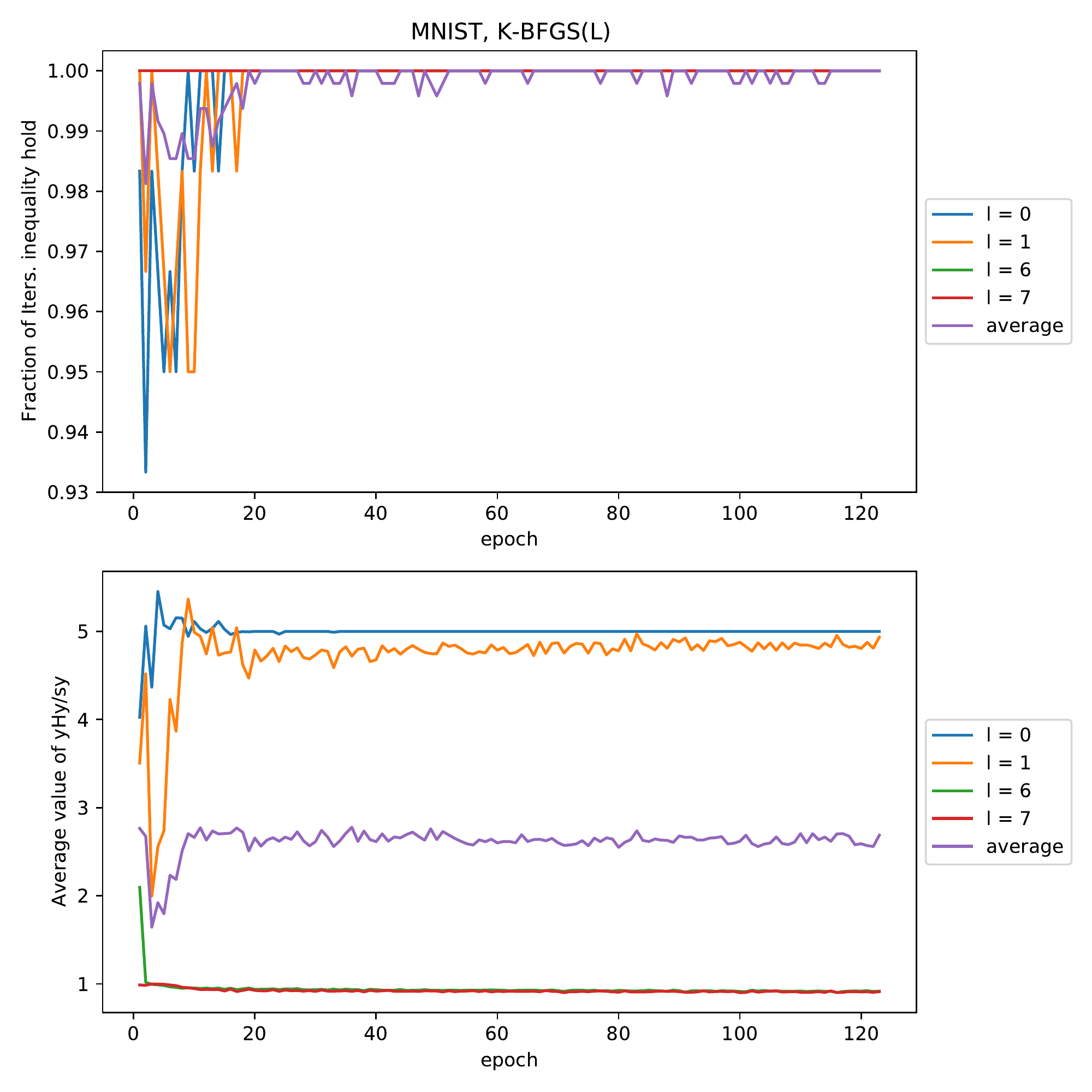}
  \end{minipage}%
  \begin{minipage}{.5\textwidth}
    \includegraphics[width=\textwidth]{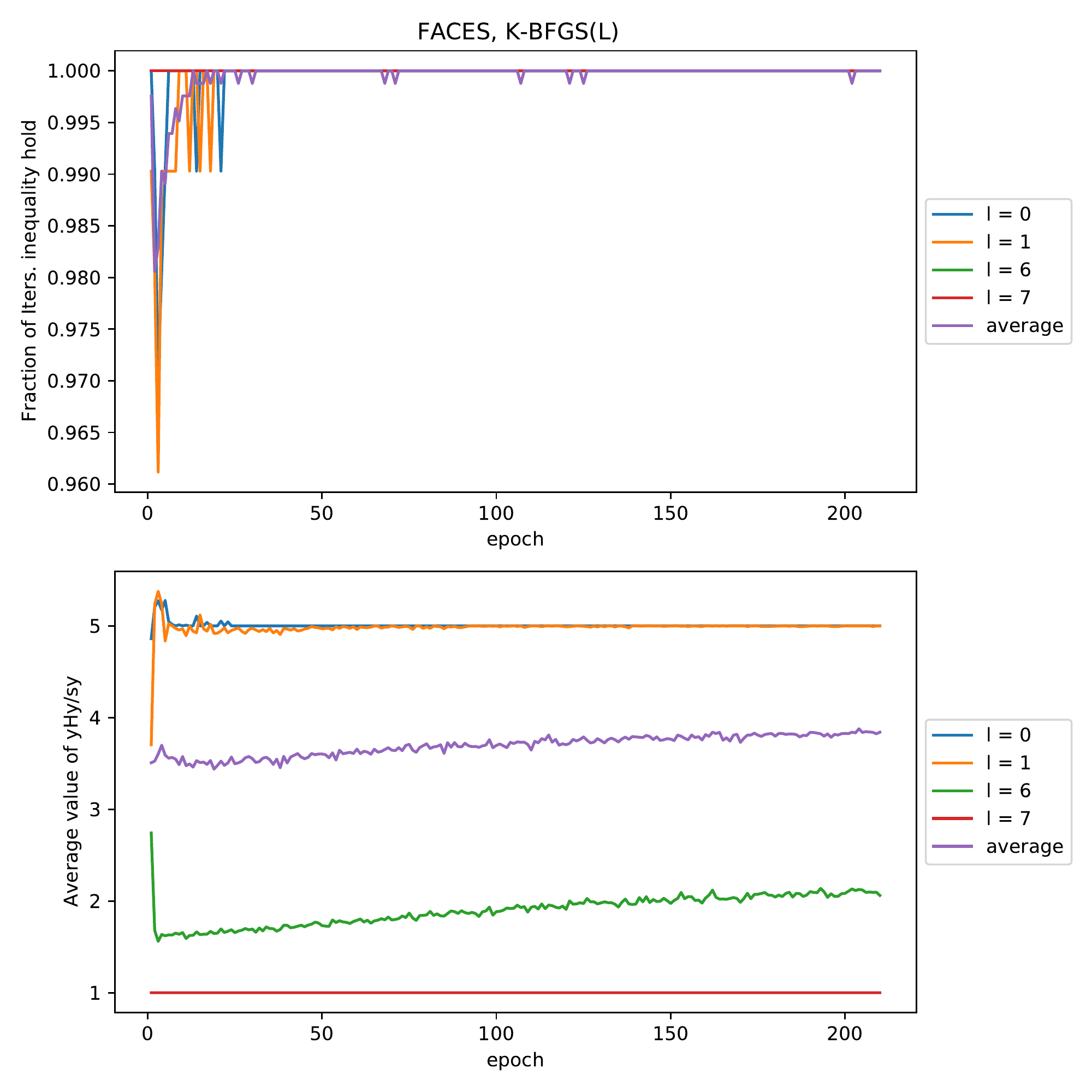}
  \end{minipage}
  \begin{minipage}{.5\textwidth}
    \includegraphics[width=\textwidth]{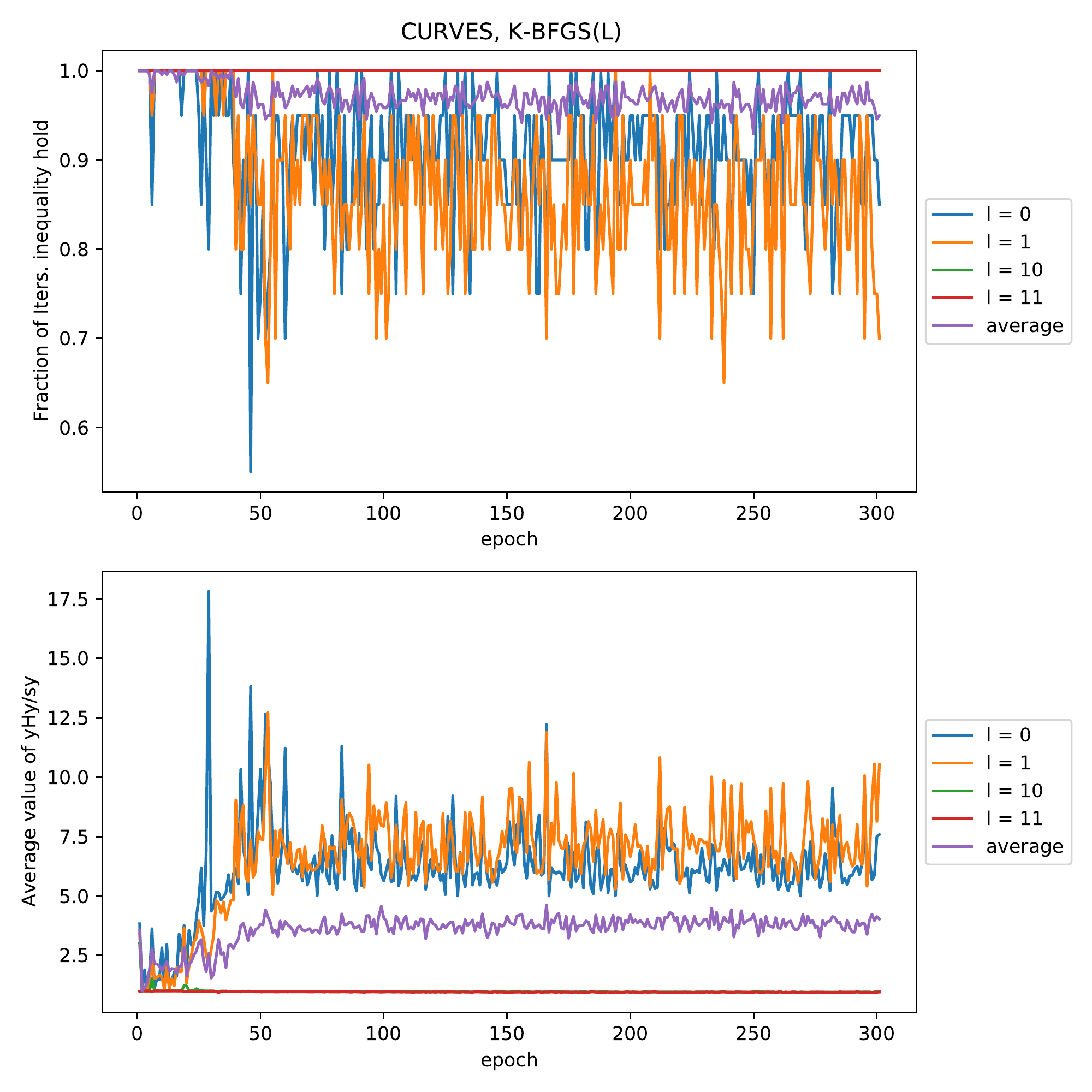}
  \end{minipage}
  
  \caption{
  Fraction of the number of iterations in each epoch, in which the inequality
  {$\frac{\vy^\top H \vy}{\vs^\top \vy} \le \frac{2}{\mu_1}$} holds (upper plots), and the average value of $\frac{\vy^T H \vy}{\vs^\top \vy}$ (lower plots) in each epoch.
  Legends in each plot assign different colors to represent each layer $l$. }
  \label{fig_16}
\end{figure}

  

Recall that there is no guarantee that
{$\frac{\vy^\top H \vy}{\vs^\top \vy} \le \frac{2}{\mu_1}$} holds after DD. While we skip using pairs that do not satisfy this inequality, when updating $H_g^l$ in proving the convergence of the K-BFGS(L) variant Algorithm \ref{klbfgs_2} , we use all $(\vs, \vy)$ pairs to update
$H_g^l$ in our implementations of both K-BFGS and K-BFGS(L) . However, whether one skips or not makes only slight difference in the performance of these algorithms, because as our empirical testing has shown, at least 90\% of the iterations satisfy
{$\frac{\vy^\top H \vy}{\vs^\top \vy} \le \frac{2}{\mu_1}$}, even if we don't skip. See Figure \ref{fig_16} which reports results on this for K-BFGS(L) when tested on the MNIST, FACES and CURVES datasets. 

\section{Implementation Details {and More Experiments}}
\label{sec_11}

\subsection{Description of Competing Algorithms}

\subsubsection{KFAC}

We first describe KFAC in Algorithm \ref{kfac}. Note that $G_l$ in KFAC refers to the $G$ matrices in \cite{martens2015optimizing}, which is different from the $G_l$ in K-BFGS. 

        
\begin{algorithm}
    \caption{{\bf KFAC}}
    \label{kfac}
    \begin{algorithmic}[1]

    \REQUIRE
    Given $\theta_0$, batch size {$m$}, and learning rate
    {$\alpha$}, {damping value $\lambda$}, inversion frequency $T$

    \FOR {$k=1,2,\ldots$}
        \STATE Sample mini-batch of size $m$: $M_k = \{\xi_{k,i}, i =1, \dots, m\}$ 
        
        \STATE Perform a forward-backward pass over the current mini-batch $M_k$ (see Algorithm \ref{feedforward})
        
        \FOR {$l=1,2,\ldots L$}
            
            
            \STATE
            $p_l = H_g^{l} \widehat{\mathbf{\nabla f}}_l H_a^l$
            \label{line_5}
            

            \STATE 
            {$W_l = W_l - \alpha \cdot p_{l}$}.
        \ENDFOR
        
        \STATE Perform another pass over $M_k$ with 
        $y$ sampled from the predictive distribution
        \label{line_2}
        
        \STATE Update
        $A_l = \beta \cdot A_l + (1-\beta) \cdot \overline{\va_{l-1} \va_{l-1}^\top}$,
        $G_l = \beta \cdot G_l + (1-\beta) \cdot \overline{\vg_{l} \vg_{l}^\top}$ 
        \label{line_3}
        
        \IF {$k \le T$ or $k \equiv 0 \pmod{T}$}
            
            \STATE 
            Recompute $H_a^l = (A_l + \sqrt{\lambda} I)^{-1}$, $H_g^l = (G_l + \sqrt{\lambda} I)^{-1}$
            \label{line_4}
        
        \ENDIF

        \ENDFOR  
    
    \end{algorithmic}
\end{algorithm}

\subsubsection{Adam/RMSprop}

We implement Adam and RMSprop exactly as in \cite{kingma2014adam} and \cite{hinton2012neural}, respectively. Note that the only difference between them is that Adam does bias correction for the 1st and 2nd moments of gradient while RMSprop does not.


\subsubsection{Initialization of Algorithms}

We now describe how each algorithm is initialized. 
For all algorithms, $\widehat{\mathbf{\nabla f}}$ is always initialized as zero.

For second-order information, we use a "warm start" to estimate the curvature when applicable. In particular, we  estimate the following curvature information using the entire training set before we start updating parameters. The information gathered is
        \begin{itemize}
        \item $A_l$ for K-BFGS and K-BFGS(L);
        \item $A_l$ and $G_l$ for KFAC;
        \item $\nabla f \odot \nabla f$ for RMSprop;
        \item Not applicable to Adam because of the bias correction.
        \end{itemize}


Lastly, for K-BFGS and K-BFGS(L), $H_a^l$ is always initially set to an identity matrix. $H_g^l$ is also initially set to an identity matrix in K-BFGS; for K-BFGS(L), when updating $H_g^l$ using L-BFGS, the incorporation of the information from the $p$ $(\vs,\vy)$ pairs is applied to an initial matrix that is set to an identity matrix. Hence, the above initialization/warm start costs are roughly twice as large for KFAC as they are for K-BFGS and K-BFGS(L).

\subsection{Autoencoder Problems}

\label{sec_12}

Table \ref{table_2} lists information about the three datasets, {namely,
MNIST\footnote{
Downloadable at \url{http://yann.lecun.com/exdb/mnist/}
},
FACES\footnote{
Downloadable at \url{www.cs.toronto.edu/~jmartens/newfaces_rot_single.mat} 
}, and
CURVES\footnote{
Downloadable at \url{www.cs.toronto.edu/~jmartens/digs3pts_1.mat}
}}.
Table \ref{table_1} specifies the architecture of the 3 problems, {where binary entropy $\mathcal{L} \left(a_{L}, y\right) = \sum_n [y_n \log a_{L,n} + (1-y_n) \log (1-a_{L,n})]$, MSE $\mathcal{L} \left(a_{L}, y\right) = \frac{1}{2} \sum_n (a_{L,n} - y_n)^2$}. Besides the loss function in Table \ref{table_1}, we further add a regularization term $\frac{\eta}{2} ||\theta||^2$ to the loss function, where $\eta = 10^{-5}$. 

\begin{table}[ht]
  \caption{Info for 3 datasets}
  \label{table_2}
  \centering
  \begin{tabular}{cccc}
    \toprule
    Dataset & \# data points
    & \# training examples
    & \# testing examples
    \\
    \midrule
    MNIST & 70,000 & 60,000 & 10,000
    \\
    FACES & 165,600 & 103,500 & 62,100
    \\
    CURVES & 30,000 & 20,000 & 10,000
    \\
    \bottomrule
  \end{tabular}
\end{table}

\begin{table}[ht]
  \caption{Architecture of 3 auto-encoder problems}
  \label{table_1}
  \centering
  \begin{tabular}{cll}
    \toprule
    Dataset
    & Layer width \& activation
    & Loss function
    \\
    \midrule
    MNIST
    & [784, 1000, 500, 250, 30, 250, 500, 1000, 784]
    & binary entropy
    \\
    & [ReLU, ReLU, ReLU, linear, ReLU, ReLU, ReLU, sigmoid]
    \\
    \midrule
    FACES
    & [625, 2000, 1000, 500, 30, 500, 1000, 2000, 625]
    & MSE
    \\
    & [ReLU, ReLU, ReLU, linear, ReLU, ReLU, ReLU, linear]
    \\
    \midrule
    CURVES
    & [784, 400, 200, 100, 50, 25, 6, 25, 50, 100, 200, 400, 784] 
    & binary entropy
    \\
    & [ReLU, ReLU, ReLU, ReLU, ReLU, linear, 
    \\
    & ReLU, ReLU, ReLU, ReLU, ReLU, sigmoid]
    \\
    \bottomrule
  \end{tabular}
\end{table}

\subsection{Specification of Hyper-parameters}

In our experiments, we focus our tuning effort onto learning rate and damping. The range of the tuning values is listed below:
\begin{itemize}
    \item learning rate $\alpha_k = \alpha \in \{$ 1e-5, 3e-5, 1e-4, 3e-4, 1e-3, 3e-3, 1e-2, 3e-2, 1e-1, 3e-1, 1, 3, 10 $\}$. 
    
    \item damping:
    \begin{itemize}
        
        \item $\lambda$ for K-BFGS, K-BFGS(L) and KFAC: $\lambda \in \{$
        {3e-3, 1e-2, 3e-2, 1e-1, 3e-1, 1, 3} $\}$.  
        
        \item $\epsilon$ for RMSprop and Adam: $\epsilon \in \{$ 1e-10, 1e-8, 1e-6, 1e-4, 1e-3, 1e-2, 1e-1 $\}$. 
        
        \item Not applicable for SGD with momentum.
    \end{itemize}

\end{itemize}

\begin{table}[ht]
  \caption{
  Best (learning rate, damping) for Figure \ref{fig_5}
  }
  \label{table_5}
  \centering
  \begin{tabular}{ccccccc}
    \toprule
    & K-BFGS
    & K-BFGS(L)
    & KFAC
    & Adam
    & RMSprop
    & SGD-m
    \\
    \midrule
    MNIST
    & (0.3, 0.3)
    & (0.3, 0.3)
    & (1, 3)
    & (1e-4, 1e-4)
    & (1e-4, 1e-4) 
    & (0.03, -)
    \\
    FACES
    & (0.1, 0.1)
    & (0.1, 0.1)
    & (0.1, 0.1)
    & (1e-4, 1e-4)
    & (1e-4, 1e-4) 
    & (0.01, -)
    \\
    CURVES
    & (0.1, 0.01)
    & (0.3, 0.3)
    & (0.3, 0.3)
    & (1e-3, 1e-3)
    & (1e-3, 1e-3) 
    & (0.1, -)
    \\
    \bottomrule
  \end{tabular}
\end{table}


The best hyper-parameters were those that produced the lowest value of the deterministic loss function encountered at the end of every epoch until the algorithm was terminated. These values were used in Figure \ref{fig_5} and are listed in Table \ref{table_5}.
Besides the tuning hyper-parameters, we also list other fixed hyper-parameters with their values:
\begin{itemize}
    \item Size of minibatch $m$ = 1000, which is also suggested in \cite{botev2017practical}.
    
    \item Decay parameter:
    \begin{itemize}
        \item K-BFGS, K-BFGS(L): $\beta = 0.9$;
        
        \item KFAC: $\beta = 0.9$;
        
        \item RMSprop, Adam: Following the notation in \cite{kingma2014adam}, we use $\beta_1=\beta_2=0.9$;\footnote{The default value of $\beta_2$ recommended in \cite{kingma2014adam} is 0.999. Hence, we also tested $\beta_2 = 0.999$, and obtained results that were similar to those presented in Figure \ref{fig_5}  (i.e., with $\beta_2=0.9$). For the sake of fair comparison, we chose to report the results with $\beta_2=0.9$.}
        
        \item SGD with momentum: $\beta = 0.9$.
    \end{itemize}
    
    \item Other:
    \begin{itemize}
        \item $\mu_1 = 0.2$ in double damping (DD):
        
        We recommend to leave the value as default because $\mu_1$ represents the "ratio" between current and past, which is scaling invariant;

        \item
        Number of $(\vs, \vy)$ pairs stored for K-BFGS(L) $p = 100$:
        
        It might be more efficient to use a smaller $p$ for the narrow layers. We didn't investigate this for simplicity and consistency;
        
    
        \item
        Inverse frequency $T = 20$ in KFAC.
    \end{itemize}

\end{itemize}

\subsection{Sensitivity to Hyper-parameters}
\label{sec_16}

\begin{figure}[H]
\centering

\begin{minipage}{.33\textwidth}
  \centering
  \includegraphics[width=\textwidth]{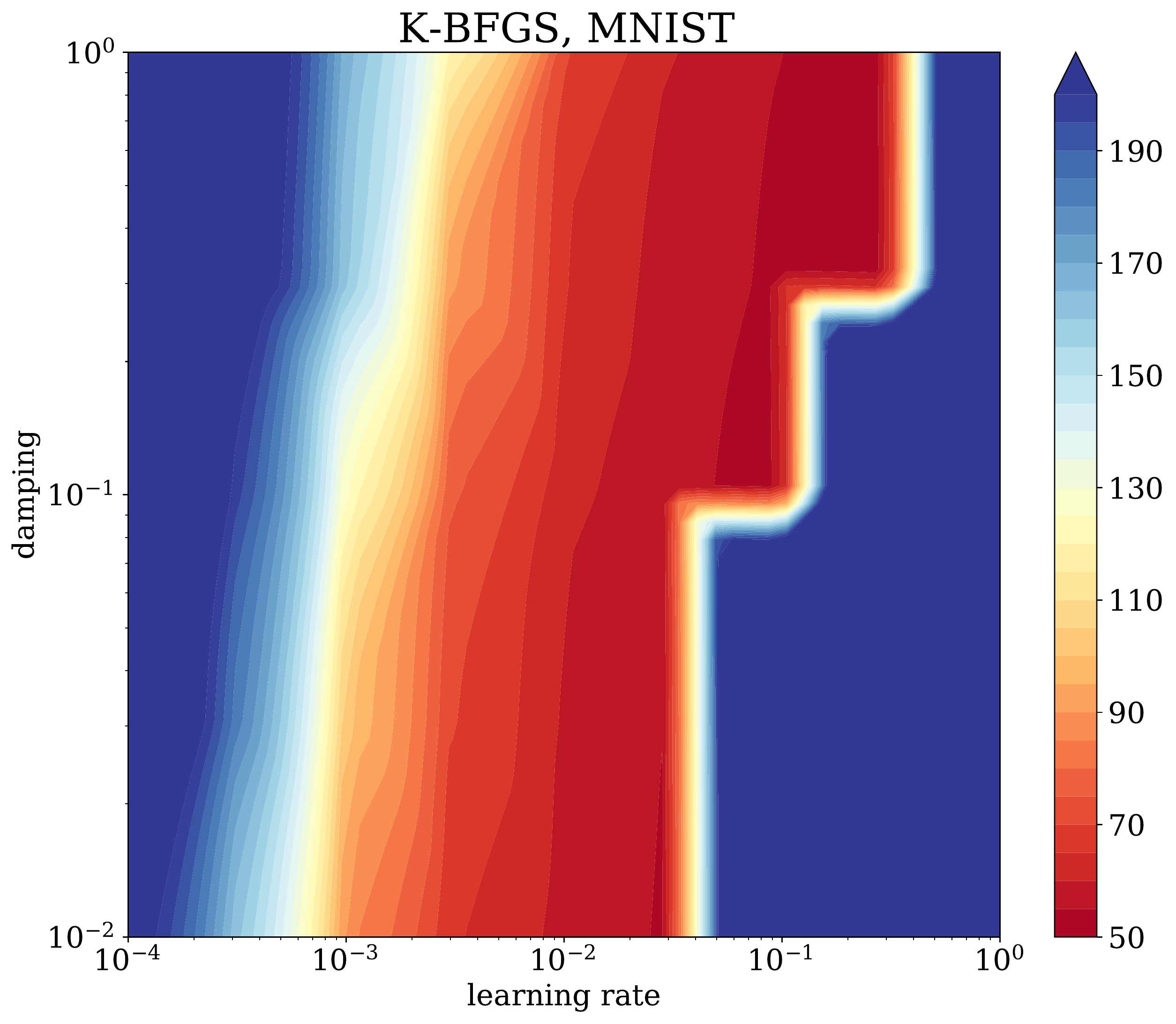}
\end{minipage}%
\begin{minipage}{.33\textwidth}
  \centering
  \includegraphics[width=\textwidth]{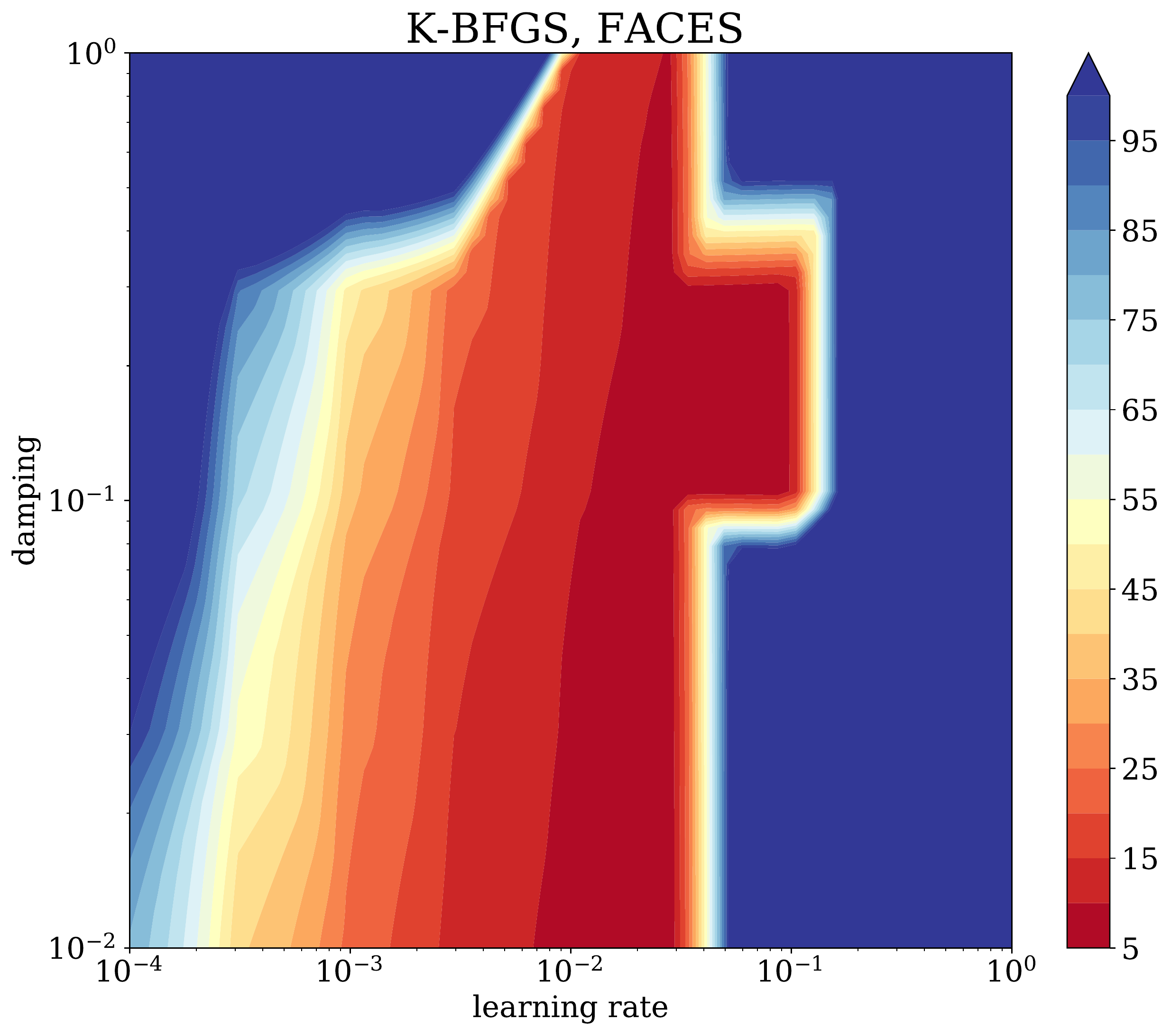}
\end{minipage}%
\begin{minipage}{.33\textwidth}
  \centering
  \includegraphics[width=\textwidth]{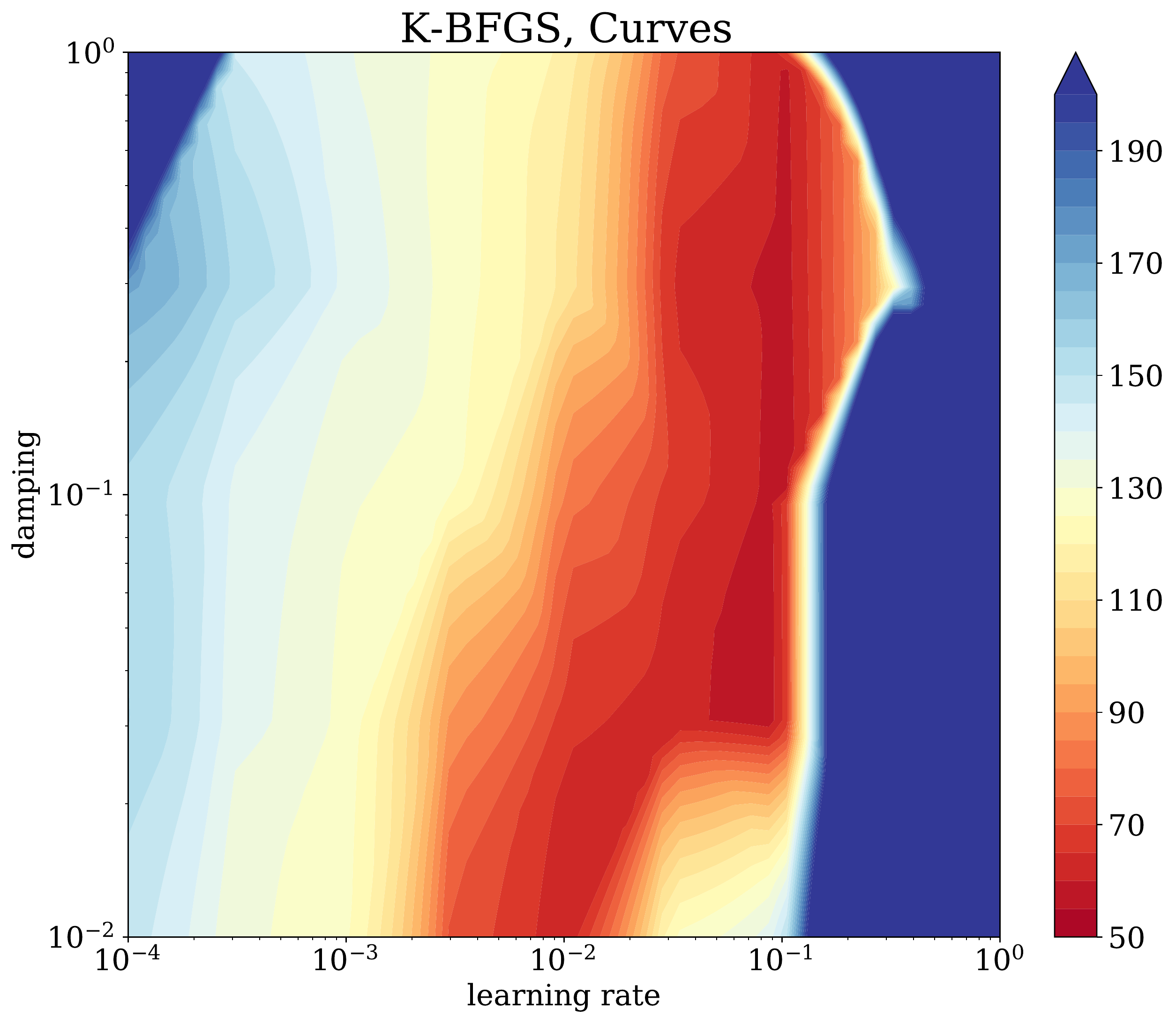}
\end{minipage}

\begin{minipage}{.33\textwidth}
  \centering
  \includegraphics[width=\textwidth]{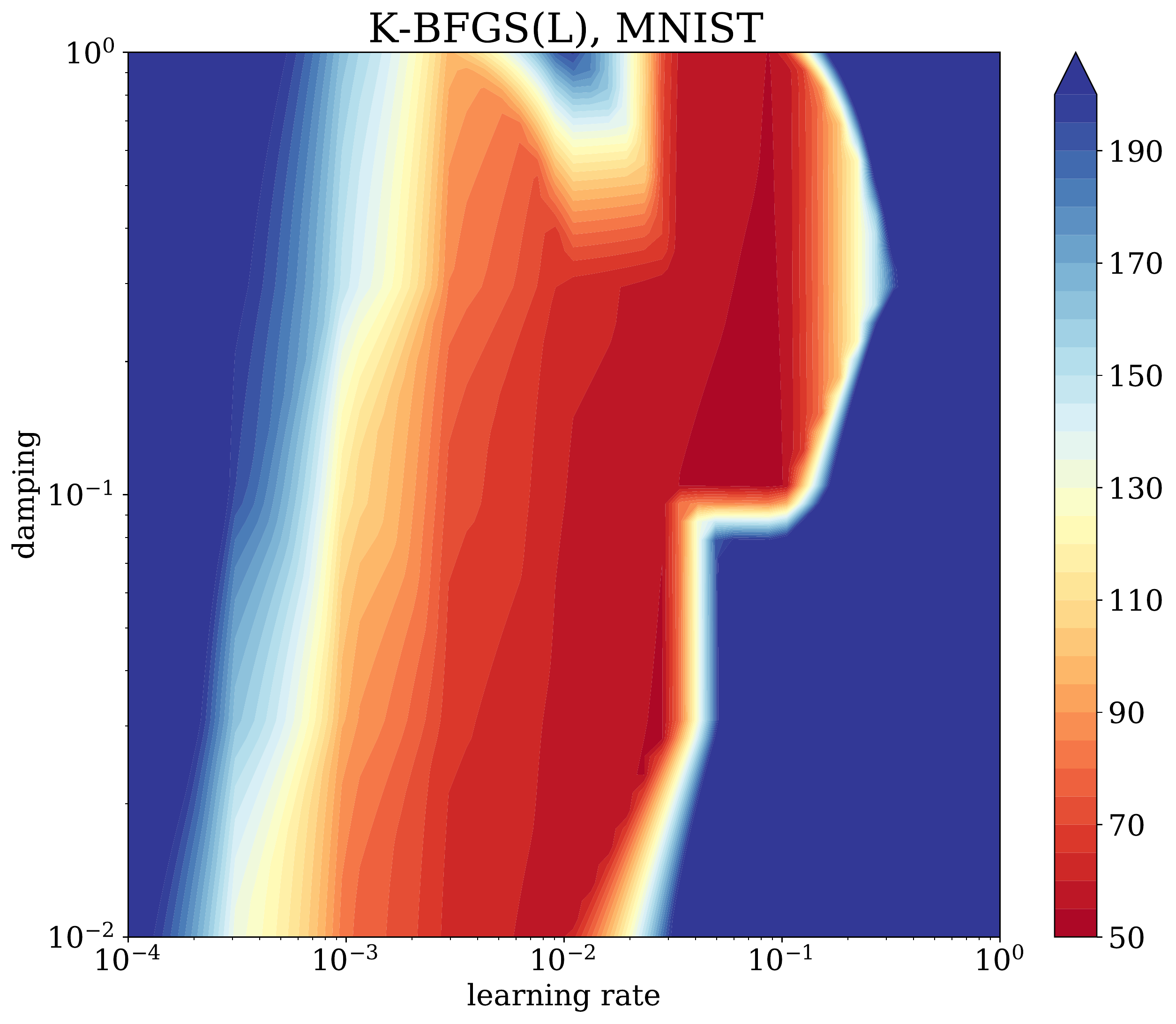}
\end{minipage}%
\begin{minipage}{.33\textwidth}
  \centering
  \includegraphics[width=\textwidth]{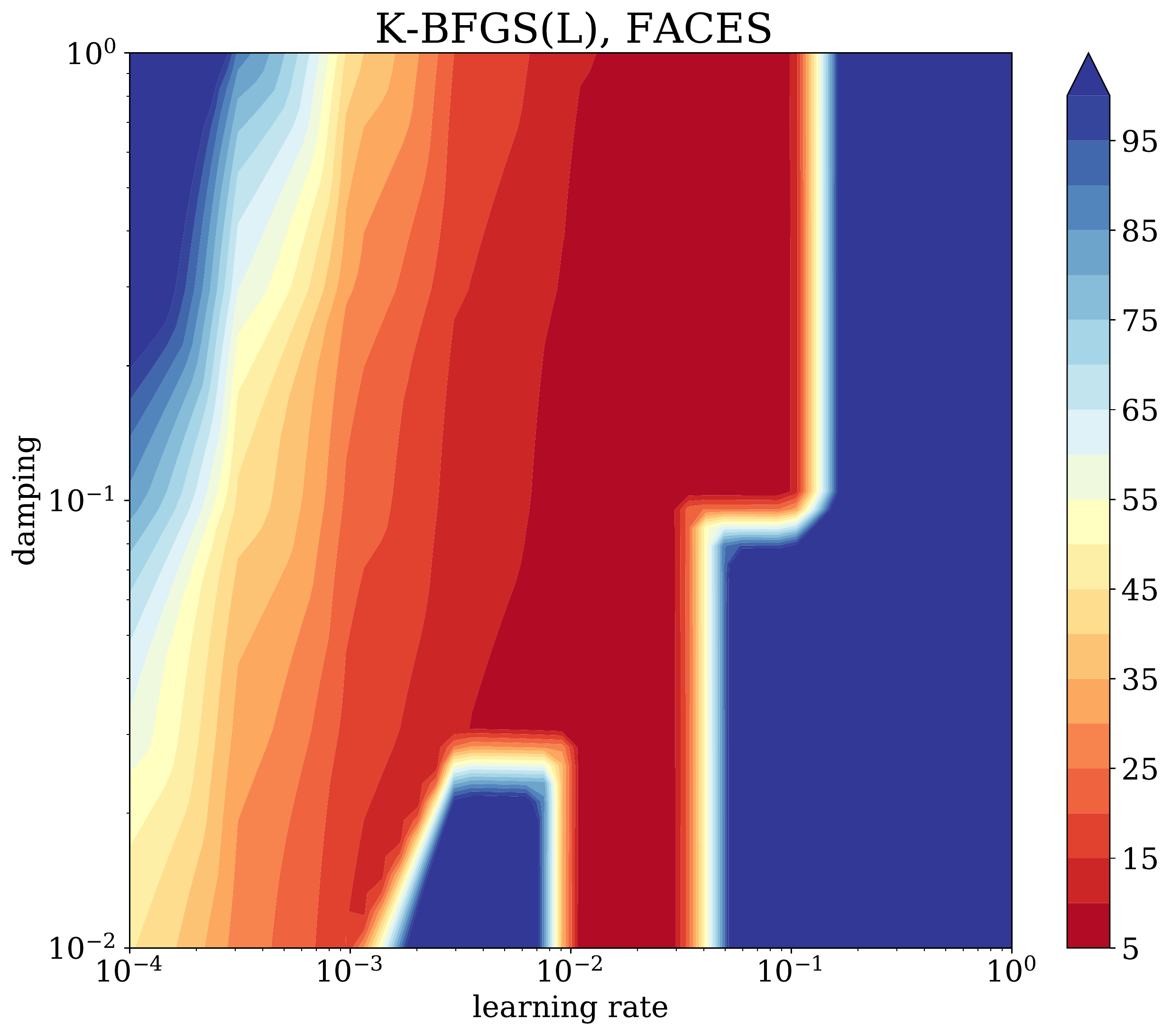}
\end{minipage}%
\begin{minipage}{.33\textwidth}
  \centering
  \includegraphics[width=\textwidth]{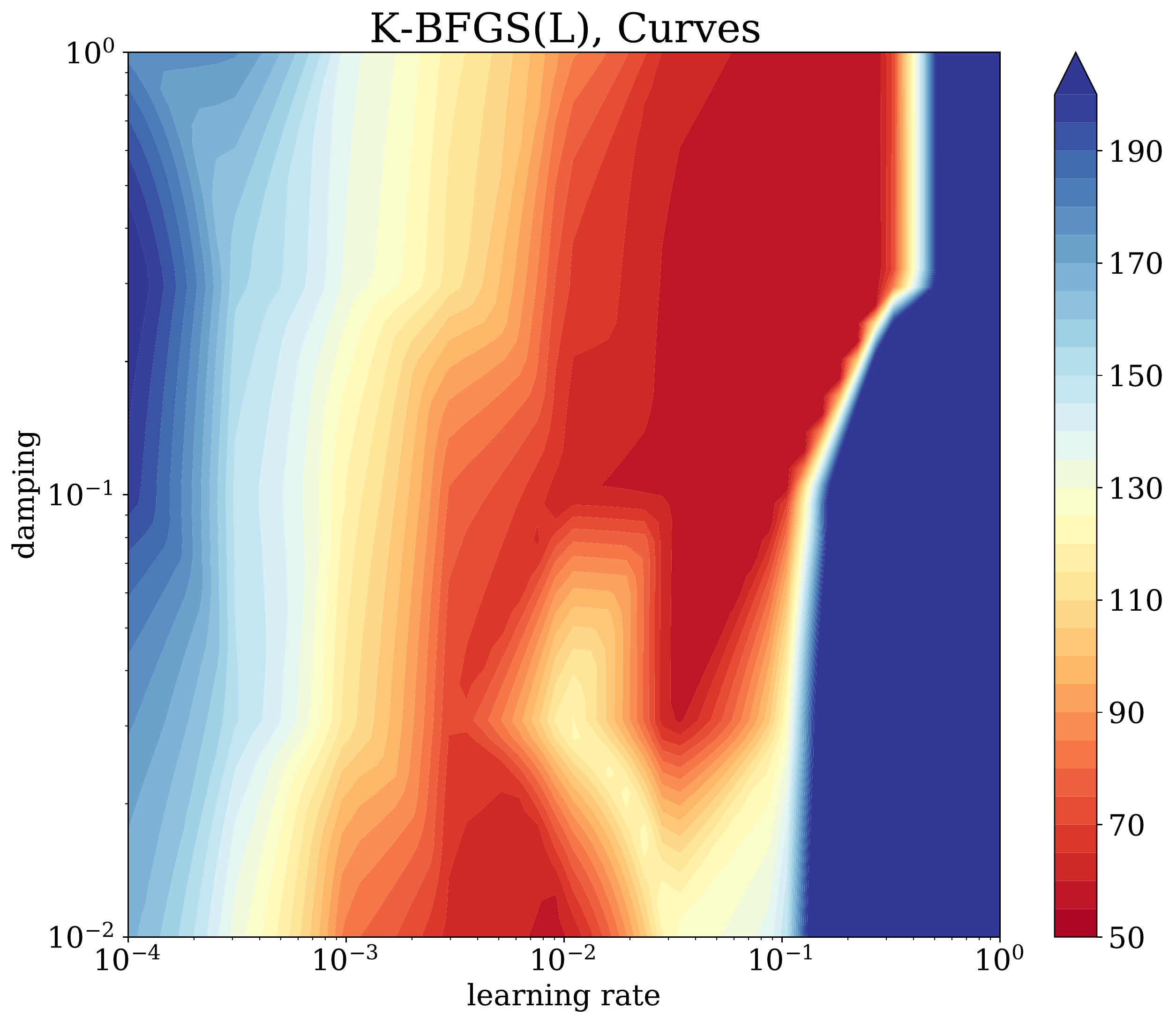}
\end{minipage}

\caption{
Landscape of loss w.r.t hyper-parameters (i.e. learning rate and damping). The left, middle, right columns depict results for MNIST, FACES, CURVES, which are terminated after 500, 2000, 500 seconds (CPU time), respectively, for K-BFGS (upper) and K-BFGS(L) (lower) row. 
}
\label{fig_15}
\end{figure}

Figure \ref{fig_15}
shows the sensitivity of K-BFGS and K-BFGS(L) to hyper-parameter values (i.e. learning rate and damping). The $x$-axis corresponds to the learning rate $\alpha$, while the $y$-axis correspond to the damping value $\lambda$.
Color corresponds to the loss after a certain amount of CPU time. We can see that both K-BFGS and K-BFGS(L) are robust within a fairly wide range of hyper-parameters. 

To get the plot, we first obtained training loss with $\alpha \in \{$1e-4, 3e-4, 1e-3, 3e-3, 1e-2, 3e-2, 1e-1, 3e-1, 1$\}$ and $\lambda \in \{$1e-2, 3e-2, 1e-1, 3e-1, 1$\}$, {and then drew contour lines of the loss within the above ranges.}





{
\subsection{Experimental Results Using Mini-batches of Size 100}

{We repeated our experiments using mini-batches of size 100 for all algorithms (see Figures \ref{fig_2}, \ref{fig_3}, and \ref{fig_4}). For each figure, the upper (lower) rower depict training loss (testing (mean square) error), whereas
  the left (right) column depicts training/test progress versus epoch (CPU time), respectively.
}



The best hyper-parameters were those that produce the lowest value of the deterministic loss function encountered at the end of every epoch until the algorithm was terminated. These values were used in Figures \ref{fig_2}, \ref{fig_3}, \ref{fig_4} and are listed in Table \ref{table_6}.

{
Our proposed methods continue to demonstrate advantageous performance, both in training and testing.
It is interesting to note that, whereas for a minibatch size of 1000, KFAC slightly outperformed K-BFGS(L), for a minibatch size of 100, K-BFGS(L) clearly outperformed KFAC in training on CURVES. }

\begin{table}[ht]
  \caption{Best (learning rate, damping) for Figures \ref{fig_2}, \ref{fig_3}, \ref{fig_4}}
  \label{table_6}
  \centering
  \begin{tabular}{ccccccc}
    \toprule
    & K-BFGS
    & K-BFGS(L)
    & KFAC
    & Adam
    & RMSprop
    & SGD-m
    \\
    \midrule
    MNIST
    & (0.1, 0.3)
    & (0.1, 0.3)
    & (0.1, 0.3)
    & (1e-4, 1e-4)
    & (1e-4, 1e-4) 
    & (0.03, -)
    \\
    FACES
    & (0.03, 0.03)
    & (0.03, 0.3)
    & (0.03, 0.3)
    & (3e-5, 1e-4)
    & (3e-5, 1e-4) 
    & (0.01, -)
    \\
    CURVES
    & (0.3, 1)
    & (0.3, 0.3)
    & (0.03, 0.1)
    & (3e-4, 1e-4)
    & (3e-3, 1e-4) 
    & (0.03, -)
    \\
    \bottomrule
  \end{tabular}
\end{table}

\begin{figure}[H]
  \centering
  \includegraphics[width=\textwidth]{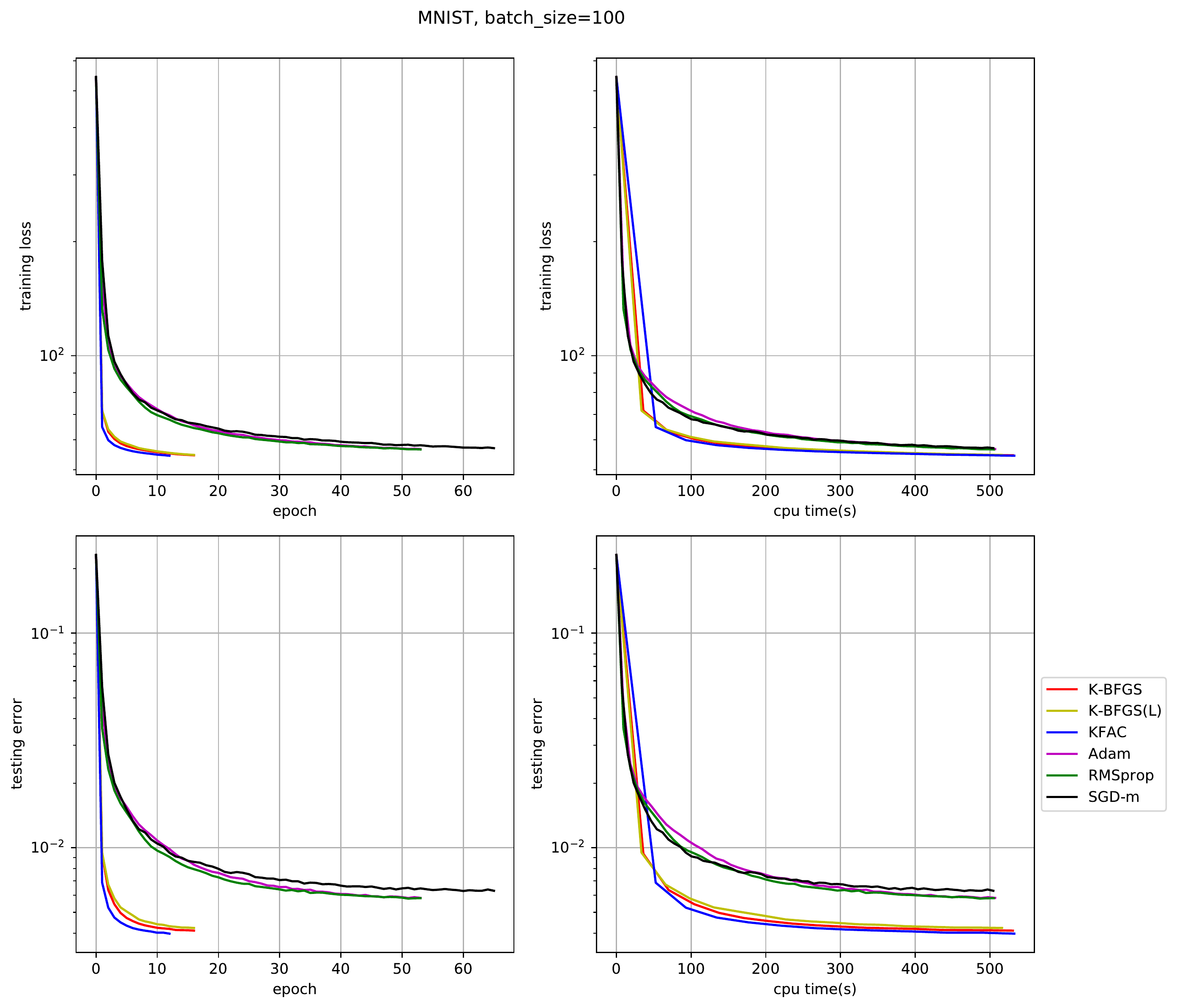}
  \caption{
  Comparison between algorithms on MNIST with batch size 100
  }
  \label{fig_2}
\end{figure}

\begin{figure}[H]
  \centering
  \includegraphics[width=\textwidth]{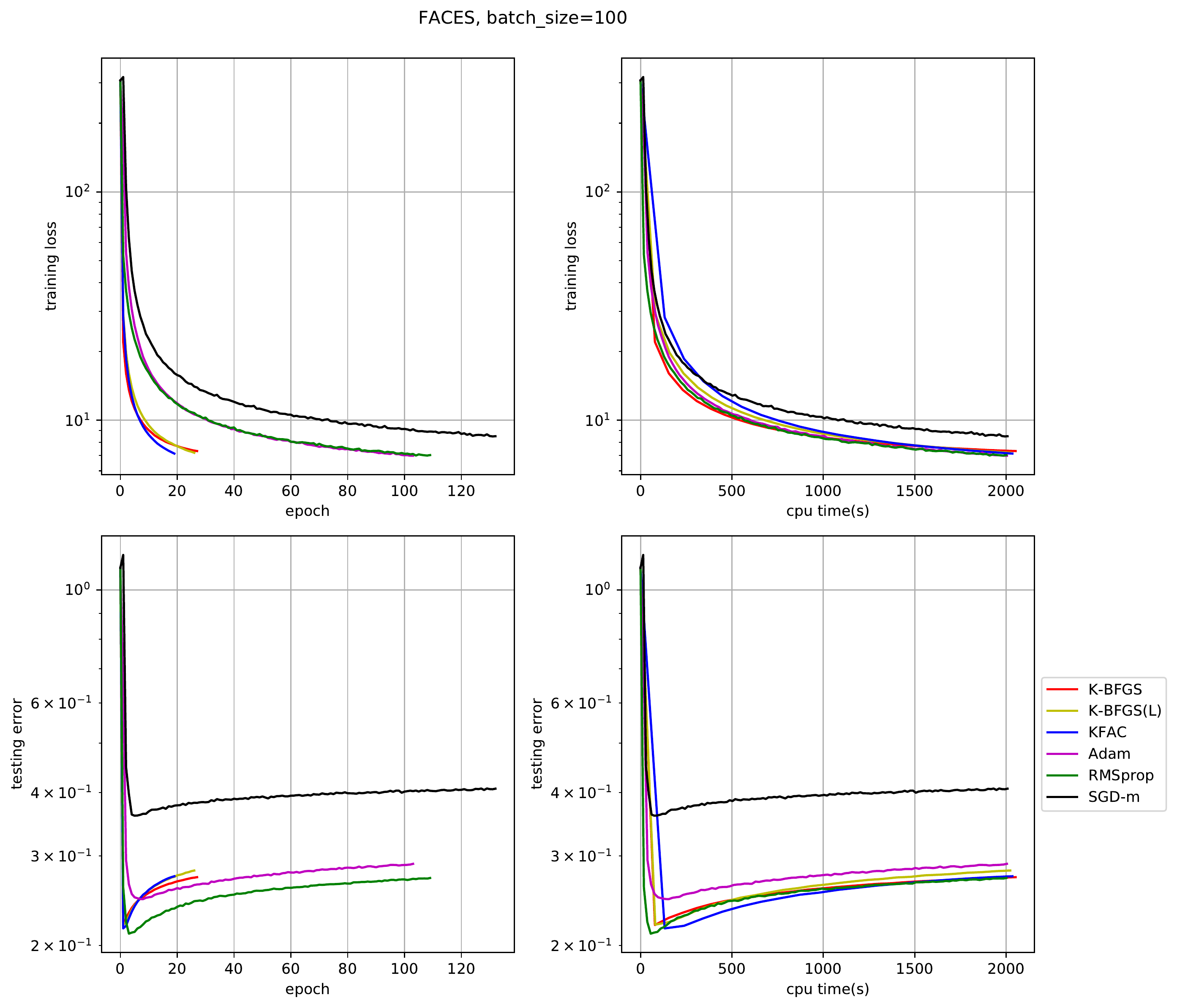}
  \caption{
  Comparison between algorithms on FACES with batch size 100
  }
  \label{fig_3}
\end{figure}

\begin{figure}[H]
  \centering
  \includegraphics[width=\textwidth]{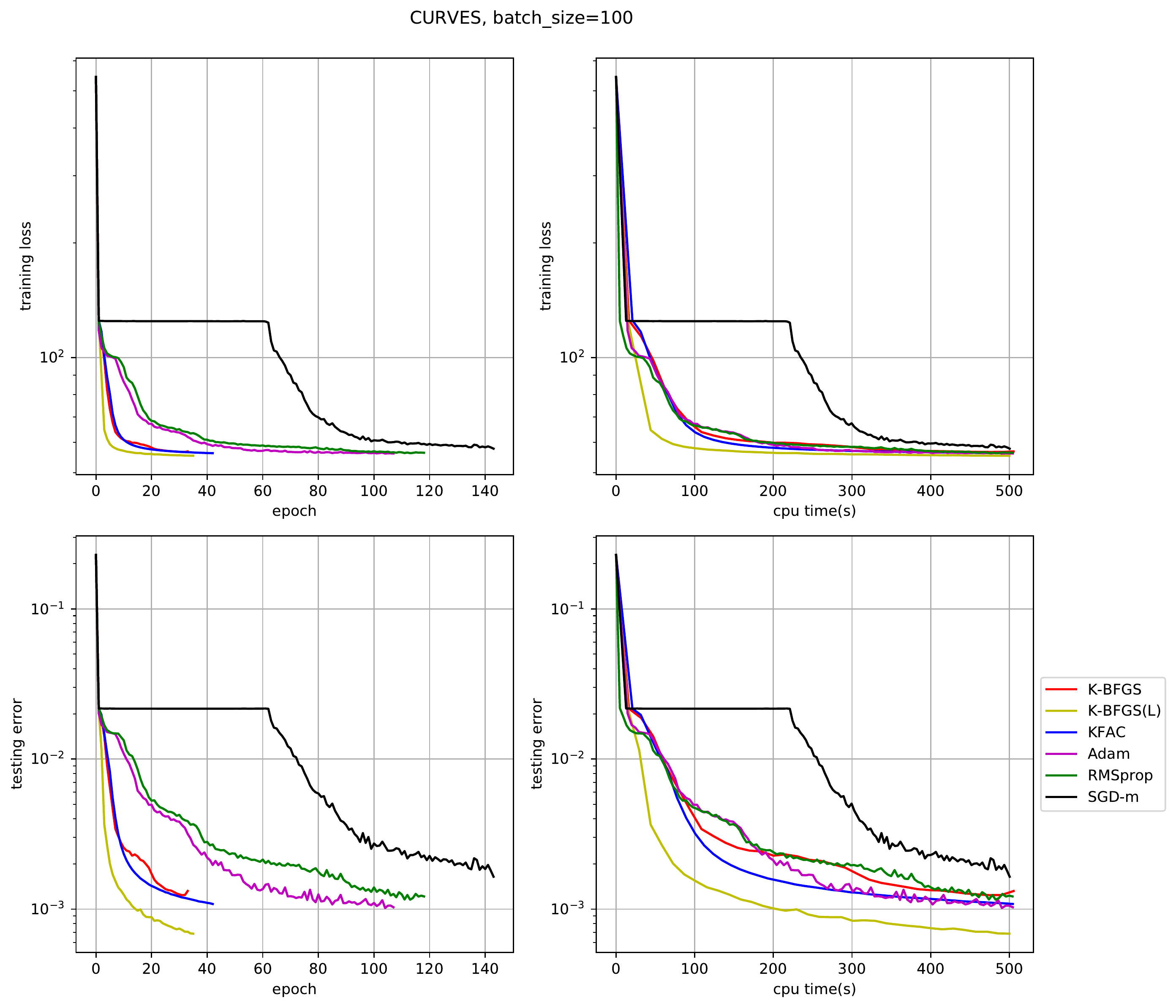}
  \caption{
  Comparison between algorithms on CURVES with batch size 100
  }
  \label{fig_4}
\end{figure}
}

\subsection{Doubling the Mini-batch for the Gradient at Almost No Cost}
\label{sec_2}


\begin{figure}[ht]
  \centering
  \includegraphics[width=\textwidth]{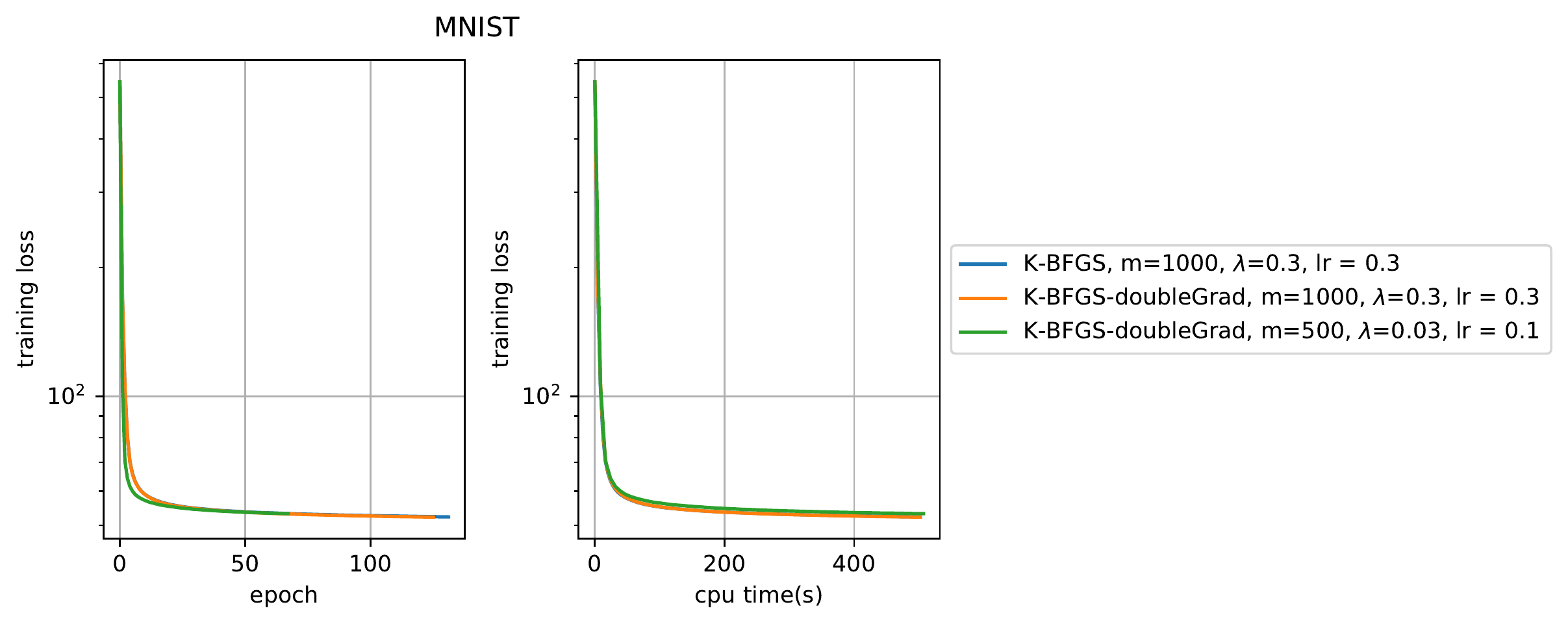}
  \includegraphics[width=\textwidth]{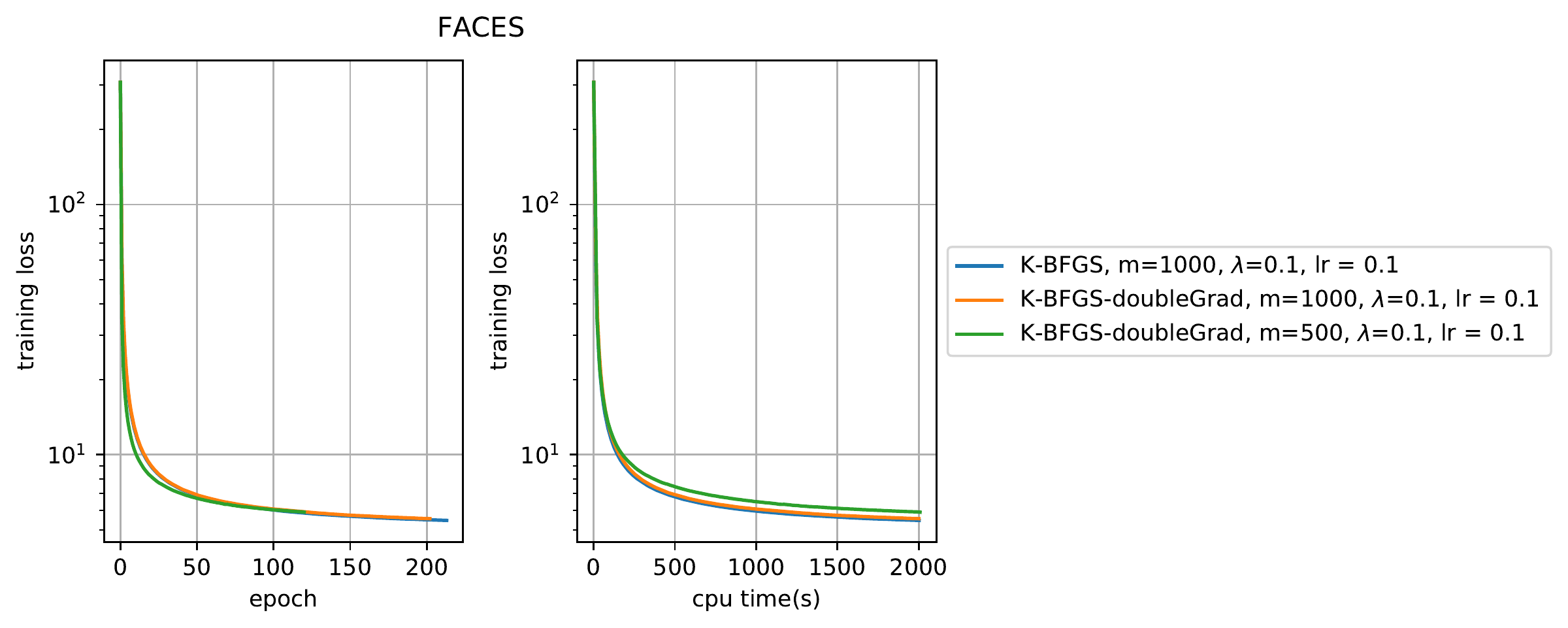}
  \includegraphics[width=\textwidth]{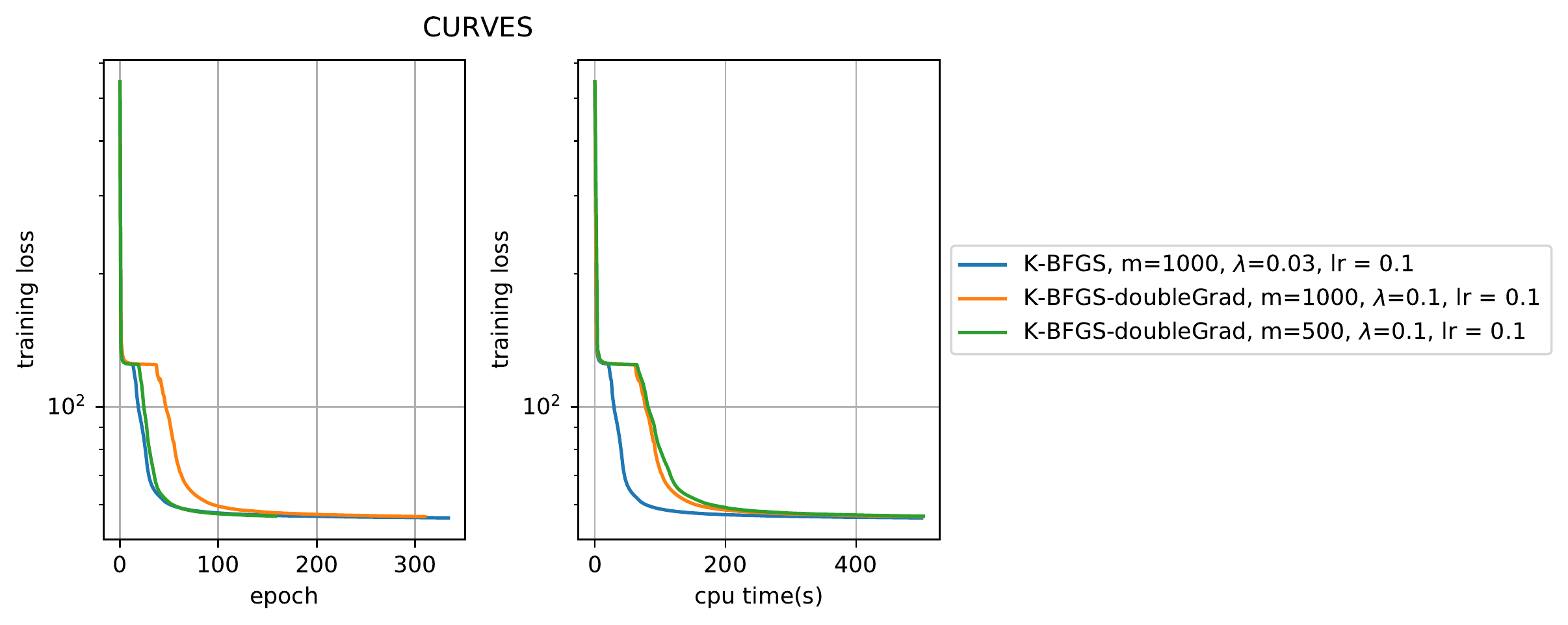}
  \caption{
Comparison between K-BFGS and its "double-grad" variants
  }
  \label{fig_12}
\end{figure}

Compared with other methods mentioned in this paper, our K-BFGS and K-BFGS(L) methods have the extra advantage of being able to double the size of the minibatch used to compute the stochastic gradient with almost no extra cost, which might be of particular interest in a highly stochastic setting. To accomplish this, we can make use of the stochastic gradient  $\overline{\mathbf{\nabla f}}^+$ computed in the \textbf{second} pass of the previous iteration that is needed for computing the $(\overline{\vs},\overline{\vy})$ pair for applying the BFGS  or L-BFGS updates, and average it with the stochastic gradient $\overline{\mathbf{\nabla f}}$ of the current iteration. For example if the size of minibatch is $m = 1000$, the above "double-grad-minibatch" method computes a stochastic gradient from 2000 data points at each iteration, except at the very first iteration.  



The results of some preliminary experiments are depicted in Figure \ref{fig_12}, where we compare
an earlier version of the K-BFGS algorithm (Algorithm \ref{algo_3}), which uses a slightly different variant of Hessian-action to update $H_a^l$,
using a size of $m=1000$ for mini-batches, with its "double-grad-minibatch" variants for $m = 500$ and  $1000$. Even though "double-grad" does not help much in these experiments, our K-BFGS algorithm performs stably across these different variants. These results indicate that there is a potential for further improvements; e.g., a finer grid search might identify hyper-parameter values that 
result in better performing algorithms.


\clearpage












\newpage

\end{document}